\documentclass{article}
\usepackage[utf8]{inputenc}
\usepackage{fullpage}
\usepackage{amsmath}
\usepackage{enumerate}
\usepackage{fancyhdr}
\usepackage{color}
\usepackage{listings}
\usepackage{graphicx}
\usepackage{amssymb}
\usepackage{mathtools}

\usepackage{algorithmic}
\usepackage{subfigure}
\definecolor{lightgray}{gray}{0.5}

\usepackage{comment}
\usepackage{epsf}
\usepackage{graphicx}

\usepackage{color}

\usepackage{mathtools}
\usepackage{amsfonts}
\usepackage{amsthm}
\usepackage{amsmath, bm}
\usepackage{amssymb}

\usepackage{pgfplots}
\pgfplotsset{width=5.75cm,compat=1.9}

\usepackage{textcomp}
\usepackage{xcolor}

\DeclareSymbolFont{extraup}{U}{zavm}{m}{n}
\DeclareMathSymbol{\vardiamond}{\mathalpha}{extraup}{87}

\usepackage[ruled, vlined, lined, commentsnumbered]{algorithm2e}

\newtheorem{theorem}{Theorem}

\newtheorem{lemma}{Lemma}
\newtheorem{corollary}{Corollary}

\newtheorem{remark}{Remark}

\long\def\comment#1{}

\newcommand{\ind}[1]{\ensuremath{{\mathbf{1}\left\{ #1 \right\}}}}

\DeclareMathOperator*{\argmax}{argmax}

\newcommand{\real}{\ensuremath{\mathbb{R}}}

\title{\textbf{Problem-Complexity Adaptive Model Selection for Stochastic Linear Bandits}}
\author{Avishek Ghosh, Abishek Sankararaman and Kannan Ramchandran \vspace{2mm} \\
Dept. of Electrical Engg. and Computer Sciences, UC Berkeley \vspace{2mm} \\
email:\{avishek$\_$ghosh,abishek\}@berkeley.edu, kannanr@eecs.berkeley.edu
}

\begin{document}
\maketitle

\begin{abstract}
\noindent 
We consider the problem of \emph{model selection} for two popular stochastic linear bandit settings, and propose algorithms that adapts to the \emph{unknown} problem complexity. In the first setting, we consider the $K$ armed mixture bandits, where the mean reward of arm $i \in [K]$\footnote{By $[r]$, we denote the set of positive integers $\{1,2,\ldots,r\}$.}, is $\mu_i+ \langle \alpha_{i,t},\theta^* \rangle $, with $\alpha_{i,t} \in \mathbb{R}^d$ being the known context vector and $\mu_i \in [-1,1]$ and $\theta^*$ are unknown parameters. 
We define\footnote{Thoroughout the paper we use $\|.\|$ to denote the $\ell_2$ norm unless otherwise specified.} $\|\theta^*\|$ as the problem complexity and consider a sequence of nested hypothesis classes, each positing a different upper bound on $\|\theta^*\|$. Exploiting this, we propose Adaptive Linear Bandit ({\ttfamily ALB}), a novel phase based algorithm that adapts to the true problem complexity, $\|\theta^*\|$.
We show that {\ttfamily ALB} achieves regret scaling of\footnote{The notation $\widetilde{\mathcal{O}}$ hides the logarithmic dependence.} $\widetilde{O}(\|\theta^*\|\sqrt{T})$, where $\|\theta^*\|$ is apriori unknown. 
As a corollary, when $\theta^*=0$, {\ttfamily ALB} recovers the minimax regret for the simple bandit algorithm without such knowledge of $\theta^*$. {\ttfamily ALB} is the first algorithm that uses parameter norm as model section criteria for linear bandits. Prior state of art algorithms \cite{osom} achieve a regret of $\widetilde{O}(L\sqrt{T})$, where $L$ is the upper bound on $\|\theta^*\|$, fed as an input to the problem. 
In the second setting, we consider the standard linear bandit problem (with possibly an infinite number of arms) where the sparsity of $\theta^*$, denoted by  $d^* \leq d$, is unknown to the algorithm. Defining $d^*$ as the problem complexity (similar to \cite{foster_model_selection}),
we show that  {\ttfamily ALB} achieves $\widetilde{O}(d^*\sqrt{T})$ regret, matching that of an oracle who knew the true sparsity level. 
This is the first algorithm that achieves such model selection guarantees.
This is methodology is then extended to the case of finitely many arms and similar results are proven.
We further verify through synthetic and real-data experiments that the performance gains are fundamental and not artifacts of mathematical bounds. In particular, we show $1.5-3$x drop in cumulative regret over non-adaptive algorithms.
\end{abstract}

\section{Introduction}
\label{sec:intro}


We study {model selection} for MAB, which refers to choosing the appropriate hypothesis class, to model the mapping from arms to expected rewards.
Model selection for MAB plays an important role in applications such as personalized recommendations, as we explain in the sequel.
Formally,  a family of nested hypothesis classes $\mathcal{H}_f$, $f \in \mathcal{F}$ needs to be specified, where each class posits a plausible model for mapping arms to expected rewards.
The true model is assumed to be contained in the family $\mathcal{F}$ which is totally ordered, where if $f_1 \leq f_2$, then $\mathcal{H}_{f_1}\subseteq \mathcal{H}_{f_2}$. 
Model selection guarantees then refers to algorithms whose regret scales in the complexity of the \emph{smallest hypothesis class containing the true model}, even though the algorithm was not aware apriori.


We consider two canonical settings for the stochastic MAB problem.
The first is the \emph{$K$ armed mixture MAB} setting, in which the mean reward from any arm $i \in [K]$ is given by $\mu_i + \langle \theta^*,\alpha_{i,t} \rangle $, where $\alpha_{i,t} \in \mathbb{R}^d$ is the known context vector of arm $i$ at time $t$, and $\mu_i \in \mathbb{R}$, $\theta^* \in \mathbb{R}^d$ are unknown and needs to be estimated.
This setting also contains the standard MAB \cite{lai-robbins,auer2002finite} when $\theta^* = 0$.
Popular linear bandit algorithms, like LinUCB, OFUL (see \cite{chu2011contextual,dani2008stochastic,oful}) handle the case with no bias ($\mu_i=0$), while OSOM \cite{osom}, the recent improvement can handle arm-bias. Implicitly, all the above algorithms assume an upper bound on the norm of $\|\theta^*\| \leq L$, which is supplied as an input.
Crucially however, the regret guarantees scale linearly in the upper bound $L$.
In contrast, we choose $\|\theta^*\|$ as the problem complexity, and provide a novel phase based algorithm, that, without any upper bound on the norm $\|\theta^*\|$,  \emph{adapts to the true complexity} of the problem instance, and achieves a regret scaling linearly in the true norm $\|\theta^*\|$.
As a corollary, our algorithm's performance matches the minimax regret of simple MAB when $\theta^* = 0$, even though the algorithm did not apriori know that $\theta^*=0$.
Formally, we consider a continuum of hypothesis classes, with each class positing a different upper bound on the norm $\|\theta^*\|$, where the complexity of a class is the upper bound posited.
As our regret bound scales linearly in $\|\theta^*\|$ (the complexity of the smallest hypothesis class containing the instance) as opposed to an upper bound on $\|\theta^*\|$, our algorithm achieves model selection guarantees.

The second setting we consider is the standard linear stochastic bandit \cite{oful} with possibly an infinite number of arms, where the mean reward of any arm $x \in \mathbb{R}^d$ (arms are vectors in this case) given by $\langle x,\theta^* \rangle$, where $\theta^* \in \mathbb{R}^d$ is unknown.
For this setting, we consider model selection from among a total of $d$ different hypothesis classes,  with each class positing a different cardinality for the support of $\theta^*$.
We exhibit a novel algorithm, where the regret scales linearly in the unknown cardinality of the support of $\theta^*$. 
The regret scaling of our algorithm matches that of an oracle that has knowledge of the optimal support cardinality \cite{sparse_bandit1},\cite{sparse_bandit2}, thereby achieving model selection guarantees.
Our algorithm is the first known algorithm to obtain regret scaling matching that of an oracle that has knowledge of the true support.
This is in contrast to standard linear bandit algorithms such as \cite{oful}, where the regret scales linearly in $d$. 
We also extend this methodology to the case when the number of arms is finite and obtain similar regret rates matching the oracle.
Model selection with dimension as a measure of complexity was also recently studied by \cite{foster_model_selection}, in which the classical contextual bandit \cite{chu2011contextual} with a finite number of arms was considered. We clarify here that although our results for the finite arm setting yields a better (optimal) regret scaling with respect to the time horizon $T$ and the support of $\theta^*$ (denoted by $d^*$), our guarantee depends on a problem dependent parameter and thus not uniform over all instances. In contrast, the results of \cite{foster_model_selection}, although sub-optimal in $d^*$ and $T$, is uniform over all problem instances. Closing this gap is an interesting future direction.



\subsection{Our Contributions}

\textbf{1. Successive Refinement Algorithms for Stochastic Linear Bandit} - We present two novel epoch based algorithms, {\ttfamily ALB (Adaptive Linear Bandit) - Norm} and {\ttfamily ALB - Dim}, that achieve model selection guarantees for both families of hypothesis respectively.
For the $K$ armed mixture MAB setting, {\ttfamily ALB-Norm}, at the beginning of each phase, estimates an \emph{upper bound} on the norm of $\|\theta^*\|$. 
Subsequently, the algorithm assumes this bound to be true during the phase, and the upper bound is re-estimated at the end of a phase.
Similarly for the linear bandit setting, {\ttfamily ALB-Dim} estimates the support of $\theta^*$ at the beginning of each phase and subsequently only plays from this estimated support during the phase. 
In both settings, we show the estimates converge to the true underlying value \textemdash in the first case, the estimate of norm $||\theta^*||$ converges to the true norm, and in the second case, for all time after a random time with finite expectation, the estimated support equals the true support.
Our algorithms are reminiscent of successive rejects algorithm \cite{successive-rejects} for standard MAB, with the crucial difference being that our algorithm is \emph{non-monotone}.
Once rejected, an arm is never pulled in the classical successive rejects. 
In contrast, our algorithm is successive \emph{refinement} and is not necessarily monotone \textemdash a hypothesis class discarded earlier can be considered at a later point of time. 
\vspace{2mm}

\noindent \textbf{2. Regret depending on the Complexity of the smallest Hypothesis Class} - 
In the $K$ armed mixture MAB setting, {\ttfamily ALB-Norm}'s regret scale as $\widetilde{O}(\|\theta^*\|\sqrt{T})$, which is superior compared to state of art algorithms such as {\ttfamily OSOM} \cite{osom}, whose regret scales as $\widetilde{O}(L\sqrt{T})$, where $L$ is an upper bound on $\|\theta^*\|$ that is supplied as an input.
As a corollary, we get the `best of both worlds' guarantee of \cite{osom}, where if $\theta^* = 0$, our regret bound recovers known minimax regret guarantee of simple MAB.
Similarly, for the linear bandit setting with unknown support,
{\ttfamily ALB-Dim} achieves a regret of $\widetilde{O}(d^*\sqrt{T})$, where $d^* \leq d$ is the true sparsity of $\theta^*$.
This matches the regret obtained by oracle algorithms that know of the true sparsity $d^*$ \cite{sparse_bandit1,sparse_bandit2}.
We also apply our methodology to the case when there is a finite number of arms and obtain similar regret scaling as the oracle.
{\ttfamily ALB-Dim} is the first algorithm to obtain such model selection guarantees.
Prior state of art algorithm {\ttfamily ModCB} for model selection with dimension as a measure of complexity was proposed in \cite{foster_model_selection}, with a finite  set of arms, where the regret guarantee was sub-optimal compared to the oracle.
However, our regret bounds for dimension, though matches the oracle, depends on the minimum non-zero coordinate value and is thus not uniform over $\theta^*$.
Obtaining regret rates in this case that matches the oracle and is uniform over all $\theta^*$ is an interesting future work.



\vspace{2mm}

\noindent \textbf{3. Empirical Validation - } We conduct synthetic and real data experiments that demonstrate superior performance of {\ttfamily ALB} compared to state of art methods such as {\ttfamily OSOM} \cite{osom} in the mixture $K$ armed MAB setting and {\ttfamily OFUL} \cite{oful} in the linear bandit setting. We further observe, that the performance of {\ttfamily ALB} is close to that of the oracle algorithms that know the true complexity. 
This indicates that the performance gains from {\ttfamily ALB} is fundamental, and not artifacts of mathematical bounds.




\paragraph{Motivating Example:}
Our model selection framework is applicable to personalized news recommendation platforms, that recommend one of $K$ news outlets, to each of its users. 
The recommendation decisions to any fixed user, can be modeled as an instance of a MAB;
the arms are the $K$ different news outlets, and the platforms recommendation decision (to this user) on day $t$ is the arm played at time $t$.
On each day $t$, each news outlet $i$ reports a story, that can be modeled by the vectors $\alpha_{i,t}$, which can be obtained by embedding the stories into a fixed dimension vector space by some common embedding schemes.
The reward obtained by the platform in recommending news outlet $i$ to this user on day $t$ can be modeled as $\mu_i + \langle \alpha_{i,t},\theta^* \rangle$, where $\mu_i$ captures the  preference of this user to news outlet $i$ and 
the vector $\theta^*$ captures the ``interest" of the user. Thus, if a channel $i$ on day $t$, publishes a news article $\alpha_{i,t}$, that this user ``likes", then most likely the content $\alpha_{i,t}$ is ``aligned" to $\theta^*$ and have a large inner product $\langle \alpha_{i,t},\theta^* \rangle$.
Different users on the platform however may have different biases and $\theta^*$.
Some users have strong preference towards certain topics and will read content written by any outlet on this topic (these users will have a large value of $\|\theta^*\|$). Other users may be agnostic to  topics, but may prefer a particular news outlet a 
 lot (for ex. some users like fox news exclusively or CNN exclusively, regardless of the topic). These users will have low $\|\theta^*\|$.
 
 In such a multi-user recommendation application, we show that our algorithm {\ttfamily ALB-Norm} that tailors the model class for each user separately is more effective (lesser regret), than to, employ a (non-adaptive) linear bandit algorithm for each user.
 We further show that our algorithms are also more effective than state of art model selection algorithms such as {\ttfamily OSOM} \cite{osom}, which posits a `binary' model - users either assign a $0$ weight to topic or assign a potentially large weight to topic. 
  Furthermore the heterogeneous complexity in this application can also be captured by the cardinality of the support of $\theta^*$; different people are interested in different sub-vectors of $\theta^*$ which the recommendation platform is not aware of apriori.
In this context, our adaptive algorithm {\ttfamily ALB-Dim} that tailors to the interest of the individual user achieves better performance compared to non-adaptive linear bandit algorithms.

\section{Related Work}


 Model selection for MAB are only recently being studied~\cite{coral,ghosh2017misspecified}, with \cite{osom}, \cite{foster_model_selection} being the closest to our work.
{\ttfamily OSOM} was proposed in \cite{osom} for model selection in the $K$ armed mixture MAB from two hypothesis classes \textemdash a ``simple model" where $\|\theta^*\| = 0$, or a ``complex model", where $0 < \|\theta^*\| \leq L$.
{\ttfamily OSOM} was shown to obtain a regret guarantee of $O(\log(T))$ when the instance is simple and $\widetilde{O}(L\sqrt{T})$ otherwise.
We refine this to consider a \emph{continuum} of hypothesis classes and propose
{\ttfamily ALB-Norm}, which achieves regret $\widetilde{O}(\|\theta^*\|\sqrt{T})$, a superior guarantee (which we also empirically verify) compared to {\ttfamily OSOM}.
Model selection with dimension as a measure of complexity was recently initiated in  \cite{foster_model_selection}, where an algorithm {\ttfamily ModCB} was proposed. 
The setup considered in \cite{foster_model_selection} was that of contextual bandits \cite{chu2011contextual} with a fixed and finite number of arms.
{\ttfamily ModCB} in this setting was shown to achieve a regret scaling that is sub-optimal compared to the oracle.
In contrast, we consider the linear bandit setting with a continuum of arms \cite{oful}, and {\ttfamily ALB-Dim} achieves a regret scaling matching that of an oracle.
The continuum of arms allows {\ttfamily ALB-Dim} a finer exploration of arms, that enables it to learn the support of $\theta^*$ reliably and thus obtain regret matching that of the oracle.
However, our regret bounds depend on the magnitude of the minimum non-zero value of $\theta^*$ and is thus not uniform over all $\beta^*$. Obtaining regret rates matching the oracle that holds uniformly over all $\theta^*$ is an interesting future work.

{\ttfamily Corral} was proposed in \cite{coral}, by casting the optimal algorithm for each hypothesis class as an expert, with the forecaster's performance having low regret with respect to the best expert (best model class).
However, {\ttfamily Corral} can only handle finitely many hypothesis classes and is not suited to our setting with continuum hypothesis classes.


Adaptive algorithms for linear bandits have also been studied in different contexts from ours. The papers of \cite{locatelli,krishnamurthy2} consider problems where the arms have an unknown structure, and propose algorithms adapting to this structure to yield low regret. The paper \cite{easy_data2} proposes an algorithm in the adversarial bandit setup that adapt to an unknown structure in the adversary's loss sequence, to obtain low regret.
The paper of \cite{temporal} consider adaptive algorithms, when the distribution changes over time. 
In the context of online learning with full feedback, there have been several works addressing model selection \cite{online_mod_sel1,online_mod_sel2,online_mod_sel3,online_mod_sel4}.
In the context of statistical learning, model selection has a long line of work (for eg. \cite{vapnik_book}, \cite{massart}, \cite{lugosi_adaptive}, \cite{arlot2011margin}, \cite{cherkassky2002model} \cite{devroye_book}). 
However, the bandit feedback in our setups is much more challenging and a straightforward adaptation of algorithms developed for either statistical learning or full information to the setting with bandit feedback is not feasible.

\section{Norm as a measure of Complexity}
\label{sec:norm}
\subsection{Problem Formulation}
\label{sec:setup}

In this section, we formally define the problem. At each round $t \in [T]$, the player chooses one of the  $K$ available arms. Each arm has a context $\{\alpha_{i,t} \in \real^d\}_{i=1}^K$ that changes over time $t$. Similar to the standard stochastic contextual bandit framework, the context vectors for each arm is chosen independently of all other arms and of the past time instances. 

We assume that there exists an underlying parameter $\theta^* \in \real^d$ and biases $\{\mu_1,\ldots,\mu_K\}$ each taking value in $[-1,1]$ such that the mean reward of an arm is a linear function of the context of the arm. The reward for playing arm $i$ at time $t$ is given by,
$
    g_{i,t} = \mu_i + \langle \alpha_{i,t}, \theta^* \rangle + \eta_{i,t}, \label{eqn:lin_ban_model}
$, where $\{\eta_{i,t}\}_{t=1}^T$ are i.i.d zero mean and $\sigma$ sub-Gaussian noise. The context vector satisfies

$$\mathbb{E}[\alpha_{i,t}|\{\alpha_{j,s},\eta_{j,s}\}_{j \in [K],s \in [t-1]\}}] = 0,
$$ and 
$$\mathbb{E}[\alpha_{i,t}\alpha_{i,t}^\top|\{\alpha_{j,s},\eta_{j,s}\}_{j \in [K],s \in [t-1]\}}] \succcurlyeq \rho_{\min}\, I.$$

The above setting is popularly known as stochastic contextual bandit \cite{osom}. In the special case of $\theta^*=0$, the above model reduces to $ g_{i,t} = \mu_i + \eta_{i,t}$. Note that in this setting, the mean reward of arms are fixed, and not dependent on the context. Hence, this corresponds to a simple \emph{multi-armed bandit} setup and standard algorithms (like UCB \cite{auer2002finite}) can be used as a learning rule. At round $t$, we define $i^*_t = \mathrm{argmax}_{i \in [K]} \left[ \mu_i + \langle \theta^*, \alpha_{i,t}\rangle \right]$ as the best arm. Also let an algorithm play arm $A_t$ at round $t$. The regret of the algorithm upto time $T$ is given by,

\begin{align*}
    R(T) = \sum_{s=1}^T \left[\mu_{i^*_s} + \langle \theta^*,\alpha_{i^*_s,s}\rangle - \mu_{A_s} - \langle \theta^*,\alpha_{A_s,s} \rangle \right].
  \end{align*}
Throughout the paper, we use $C,C_1,..,c,c_1,..$ to denote positive universal constants, the value of which may differ in different instances.

 We define a new notion of complexity for stochastic linear bandits; and propose an algorithm that adapts to it. We define $\|\theta^*\|$ as the problem complexity for the linear bandit instance. Note that if $\|\theta^*\| =0$, the linear bandit model reduces to the simple multi-armed bandit setting. Furthermore, the cumulative regret $R(T)$ of linear bandit algorithms (like OFUL \cite{oful} and OSOM \cite{osom}) scales linearly with $\|\theta^*\|$ (\cite{osom}). Hence, $\|\theta^*\|$ constitutes a natural notion of model complexity. In Algorithm~\ref{algo:main_algo}, we propose an adaptive scheme which adapts to the true complexity of the problem, $\|\theta^*\|$. Instead of assuming an upper-bound on $\|\theta^*\|$, we use an initial exploration phase to obtain a rough estimate of $\|\theta^*\|$ and then successively refine it over multiple epochs. The cumulative regret of our proposed algorithm actually scales linearly with $\|\theta^*\|$.

\subsection{Adaptive Linear Bandit (norm)---({{\ttfamily ALB-Norm}} algorithm)}
\label{sec:method}
 We present the adaptive scheme in Algorithm~\ref{algo:main_algo}. Note that Algorithm~\ref{algo:main_algo} depends on the subroutine OFUL$^+$. Observe that at each iteration, we estimate the bias $\{\mu_1,\ldots,\mu_K\}$ and $\theta^*$ separately. The estimation of the bias involves a simple sample mean estimate with upper confidence level, and the estimation of $\theta^*$ involves building a confidence set that shrinks over time.

In order to estimate $\theta^*$, we use a variant of the popular OFUL \cite{oful} algorithm with arm bias. We refer to the algorithm as OFUL$^+$.  Algorithm~\ref{algo:main_algo} is epoch based, and over multiple epochs, we successively refine the estimate of  $\|\theta^*\|$. We start with a rough over-estimate of $\|\theta^*\|$ (obtained from a pure exploration phase), and based on the confidence set constructed at the end of the epoch, we update the estimate of $\|\theta^*\|$. We argue that this approach indeed correctly estimates $\|\theta^*\|$ with high probability over a sufficiently large time horizon $T$.



\begin{algorithm}[t!]
  \caption{Adaptive Linear Bandit (Norm)}
  \begin{algorithmic}[1]
 \STATE  \textbf{Input:} Initial exploration period $\tau$, the phase length $T_1$,  $\delta_1 > 0$, $\delta_s > 0$.
 \STATE Select an arm at random, sample rewards $2\tau$ times
 \STATE Obtain initial estimate ($b_1$) of $\|\theta^*\|$ according to Section~\ref{sec:initial_val}
 \FOR{$t = 1,2,\ldots,K$}
 \STATE Play arm $t$, receive reward $g_{t,t}$
 \ENDFOR
 \STATE Define $\mathcal{S} = \{g_{i,i}\}_{i=1}^K$ 
  \FOR{ epochs $i=1,2 \ldots, N $}
  \STATE Use $\mathcal{S}$ as pure-exploration reward
  \STATE Play OFUL$_{\delta_i}^+(b_{i})$ until the end of epoch $i$ (denoted by $\mathcal{E}_i$)
  \STATE At $t=\mathcal{E}_i$, refine estimate of $\|\theta^*\|$ as, $ b_{i+1} = \max_{\theta \in \mathcal{C}_{\mathcal{E}_i}} \|\theta\|$
  \STATE Set $T_{i+1} = 2 T_{i}$, $\delta_{i+1} = \frac{\delta_i}{2}$.
    \ENDFOR
    \STATE \underline{\texttt{OFUL$^+_\delta(b)$:}}
     \STATE  \textbf{Input:} Parameters $b$, $\delta >0$, number of rounds $\Tilde{T}$
 \FOR{$t = 1,2, \ldots, \Tilde{T} $}
 \STATE Select the best arm estimate as $
     j_t = \mathrm{argmax}_{i\in [K]} \left[ \max_{\theta \in \mathcal{C}_{t-1}} \{ \Tilde{\mu}_{i,t-1} + \langle \alpha_{i,t}, \theta \rangle \} \right]$, \\
     where $\Tilde{\mu}_{i,t}$ and $\mathcal{C}_t$ are given in Section~\ref{sec:method}.
 \STATE Play arm $j_t$, and update $\{\Tilde{\mu}_{i,t}\}_{i=1}^K$ and $\mathcal{C}_{t}$
 \ENDFOR
  \end{algorithmic}
  \label{algo:main_algo}
\end{algorithm}

We now discuss the algorithm OFUL$^+$. A variation of this was proposed in \cite{osom} in the context of model selection between linear and standard multi-armed bandits. We use $\Tilde{\mu}_{i,t}$ to address the bias term, which we define shortly. The parameters $b$ and $\delta$ are used in the construction of the confidence set $\mathcal{C}_t$. Suppose OFUL$^+$ is run for a total of $\Tilde{T}$ rounds and plays arm $A_s$ at time $s$. Let $T_i(t)$ be the number of times OFUL$^+$ plays arm $i$ until time $t$. Also, let $b$ be the current estimate of $\|\theta^*\|$. We define,
$$
    \Bar{g}_{i,t} = \frac{1}{T_i(t)}\sum_{s=1}^t g_{i,s} \ind{A_s =t}.
$$ With this, we have \footnote{For complete expression, see Appendix~\ref{app:oful}} $$
    \Tilde{\mu}_{i,t} = \Bar{g}_{i,t} + c (\sigma + b) \sqrt{\frac{d}{T_i(t)}\log \left(\frac{1}{\delta}\right)}. $$ The confidence interval $\mathcal{C}_t$, is defined as $$
    \mathcal{C}_t =  \lbrace \theta \in \real^d: \|\theta - \hat{\theta}_t\| \leq \mathcal{K}_\delta(b,t,\Tilde{T})  \rbrace, $$ where $\hat{\theta}_t$ is the least squares estimate defined as $$
    \hat{\theta}_t = \left( \mathbf{\alpha}^\top_{K+1:t} \mathbf{\alpha}_{K+1:t} + I \right)^{-1} \mathbf{\alpha}^\top_{K+1:t} G_{K+1:t}$$ 
with $\alpha_{K+1:t}$ as a matrix having rows $\alpha_{A_{K+1},K+1}^\top, \ldots, \alpha_{A_t,t}^\top$ and $G_{K+1:t} = [g_{A_{K+1},K+1} - \Tilde{\mu}_{A_{K+1},K+1}, \ldots, g_{A_t,t} - \Tilde{\mu}_{A_t,t}]^\top$. The radius of $\mathcal{C}_t$ is given by (see Appendix~\ref{app:oful} for complete expression), $$\mathcal{K}_\delta(b,t,\Tilde{T}) = c \frac{(\sigma \sqrt{d} + b)}{\rho_{\min} \sqrt{t}}\sqrt{\log (K \Tilde{T}/\delta)}.$$  Lemma 2 of \cite{osom} shows that $\theta^* \in \mathcal{C}_t$ with probability  \footnote{There is a typo in the proof of regret in \cite{osom}. We correct the typo, and modify the definition of $\Tilde{\mu}_{i,t}$ and $\mathcal{K}_\delta(b,t,\Tilde{T})$. As a consequence, the high probability bounds change a little.} at least $1-4\delta$. 

\subsection{Construction of initial estimate $b_1$}
\label{sec:initial_val}

We select an arm at random (without loss of generality, assume that this is arm $1$), and sample rewards (in an i.i.d fashion) for $2\tau$ times, where $\tau>0$ is a parameter to be fed to the Algorithm~\ref{algo:main_algo}. In order to kill the bias of arm $1$, we take pairwise differences and form: $ y(1) = g_{1,1} - g_{1,2}, \,\,
    y(2) = g_{1,3} - g_{1,4}
$ and so on. Augmenting $y(.)$, we obtain: $ Y = \Tilde{X}\theta^* + \Tilde{\eta},
$ where the $i$-th row of $\Tilde{X}$ is $(\alpha_{1,2i+1} - \alpha_{1,2i+2})^\top$, the $i$-th element of $\Tilde{\eta}$ is $\eta_{1,2i+1} - \eta_{1,2i+2}$. Hence, the least squares estimate, $\widehat{\theta}^{(\ell s)}$ satisfies $
    \|\widehat{\theta}^{(\ell s)} - \theta^*\| \leq \sqrt{2}\sigma\sqrt{\frac{d}{\tau}\log(1/\delta_s)}$, with probability exceeding $1-\delta_s$ (\cite{wainwright2019high}). We set the initial estimate 
    $$b_1 = \max \lbrace \|\widehat{\theta}^{(\ell s)}\| + \sqrt{2}\sigma\sqrt{\frac{d}{\tau}\log(1/\delta_s)}, \,\, 1  \rbrace $$
and this satisfies $b_1 \geq \|\theta^*\|$ and $b_1 \geq 1$ with probability at least $1-\delta_s$.


\subsection{Regret Guarantee of Algorithm~\ref{algo:main_algo}}
\label{sec:regret}
We now obtain an upper bound on the cumulative $R(T)$ with Algorithm~\ref{algo:main_algo} with high probability. For theoretical tractability, we assume that OFUL$^+$ restarts at the start of each epoch. We have the following lemma regarding the sequence $\{b_i\}_{i=1}^\infty$ of estimates of $\|\theta^*\|$:
\begin{lemma}
\label{lem:b_seq}
With probability exceeding $1-8\delta_1 - \delta_s$, the sequence $\{b_i\}_{i=1}^\infty$ converges to $\|\theta^*\|$ at a rate $\mathcal{O}(\frac{i}{2^i})$, and we obtain $
    b_{i} \leq \left( c_1 \|\theta^*\| + c_2 \right) $ for all $i$, provided $ T_1 \geq C_1 \left ( \max\{p,q\}  \, b_1 \right)^2 \, d$, where $C_1 > 9$, and $p = [\frac{14 \log (\frac{2KT_1}{\delta_1}) }{\sqrt{\rho_{\min}}} ], \quad q = [\frac{ 2C \sigma \log (\frac{2KT_1}{\delta_1})}{\sqrt{\rho_{\min}}}].$
\end{lemma}
Hence, the sequence converges to $\|\theta^*\|$ at an exponential rate. We have the following guarantee on the cumulative regret $R(T)$:
\begin{theorem}
\label{thm:regret_norm_based}
Suppose $T_1 > \max\{T_{\min}(\delta,T), C_1 \left ( \max\{p,q\}  \, b_1 \right)^2 \, d \}$, where $C_1 > 9$ and $T_{\min}(\delta,T) = ( \frac{16}{\rho^2_{\min}} + \frac{8}{3\rho_{\min}}) \log ( \frac{2dT}{\delta})$. Then, with probability at least $1-18\delta_1 -\delta_s$, we have
\begin{align*}
    R(T) \leq  C_1 (2\tau + K) \|\theta^*\| + C( \|\theta^*\| + 1) (\sqrt{K} + \sqrt{d}) \, \sqrt{T} \, \log (K T_1/\delta_1 ) \log (T/T_1).
\end{align*}
\end{theorem}

\begin{remark}
Note that the regret bound depends on the problem complexity $\|\theta^*\|$, and we prove that Algorithm~\ref{algo:main_algo} adapts to this complexity. Ignoring the log factors, Algorithm~\ref{algo:main_algo} has a regret of $\Tilde{\mathcal{O}}( (1+\|\theta^*\|)(\sqrt{K} + \sqrt{d}) \sqrt{T})$ with high probability.
\end{remark}
\begin{remark}
(Matches Linear Bandit algorithm) Note that the above bound matches the regret guarantee of the linear bandit algorithm with bias as presented in \cite{osom}.
\end{remark}
\begin{remark}
 (Matches UCB when $\theta^* = 0$) When $\theta^* = 0$ (the simplest model, without any contextual information), Algorithm~\ref{algo:main_algo} recovers the minimax regret of UCB algorithm. Indeed, substituting $\|\theta^*\|=0$ in the above regret bound yields $ R(T) = \mathcal{O}(\sqrt{K T})$, with high probability, provided $K>d$.  Hence, we obtain the ``best of both worlds'' results with simple model ($\theta^* =0$) and contextual bandit model ($\theta^* \neq 0$).
\end{remark}

\section{Dimension as a Measure of Complexity - Continuum Armed Setting}
\label{sec:dimension_adaptation}

In this section, we consider the standard stochastic linear bandit model in $d$ dimensions \cite{oful}, with the dimension as a measure of complexity. 
The setup in this section is almost identical to that in Section \ref{sec:setup}, with the $0$ arm biases and a continuum collection of arms denoted by the set $\mathcal{A}:= \{x \in \mathbb{R}^d:\|x\| \leq 1\}$\footnote{Our algorithm can be applied to any compact set $\mathcal{A} \subset \mathbb{R}^d$, including the finite set as shown in Appendix \ref{appendix-comparision}.}
Thus, the mean reward from any arm $x \in \mathcal{A}$ is $ \langle x,\theta^* \rangle $, where $\|\theta^*\| \leq 1$. 
We assume that $\theta^*$ is $d^* \leq d$ sparse, where $d^*$ is apriori unknown to the algorithm.
Thus, unlike in Section \ref{sec:norm}, there is no i.i.d. context sampling in this section.
We consider a sequence of $d$ nested hypothesis classes, where each hypothesis class $i \leq d$, models $\theta^*$ as a $i$ sparse vector. The goal of the forecaster is to minimize the regret, namely $R(T) \coloneqq \sum_{t=1}^T \left[ \langle x^*_t - x_t,\theta^* \rangle\right]$, where at any time $t$, $x_t$ is the action recommended by an algorithm and $x^*_t = \mathrm{argmax}_{x \in \mathcal{A}} \langle x,\theta^*\rangle$. The regret $R(T)$ measures the loss in reward of the forecaster with that of an oracle that knows $\theta^*$ and thus can compute $x^*_t$ at each time.

\subsection{{\ttfamily ALB-Dim} Algorithm}


The algorithm is parametrized by $T_0 \in \mathbb{N}$, which is given in Equation (\ref{eqn:T_0_defn}) in the sequel and slack $\delta \in (0,1)$.
As in the previous case, {\ttfamily ALB-Dim} proceeds in phases numbered $0,1,\cdots$ which are non-decreasing with time. 
At the beginning of each phase, {\ttfamily ALB-Dim} makes an estimate of the set of non-zero coordinates of $\theta^*$, which is kept fixed throughout the phase.
Concretely, each phase $i$ is divided into two blocks - {\em (i)} a regret minimization block lasting $25^i T_0$ time slots, {\em (ii)} followed by a random exploration phase lasting $5^i \lceil\sqrt{T_0}\rceil$ time slots.
Thus, each phase $i$ lasts for a total of $25^iT_0 + 5^i \lceil \sqrt{T_0} \rceil$ time slots.
At the beginning of each phase $i \geq 0$, $\mathcal{D}_i \subseteq [d]$ denotes the set of `active coordinates', namely the estimate of the non-zero coordinates of $\theta^*$.
Subsequently, in the regret minimization block of phase $i$, a fresh instance of OFUL \cite{oful} is spawned, with the dimensions restricted only to the set $\mathcal{D}_i$ and probability parameter $\delta_i:= \frac{\delta}{2^i}$. In the random exploration phase, at each time, one of the possible arms from the set $\mathcal{A}$ is played chosen uniformly and independently at random. 
At the end of each phase $i\geq 0$,  {\ttfamily ALB-Dim} forms an estimate $\widehat{\theta}_{i+1}$ of $\theta^*$, by solving a least squares problem using all the random exploration samples collected till the end of phase $i$.
The active coordinate set $\mathcal{D}_{i+1}$, is then the coordinates of $\widehat{\theta}_{i+1}$ with  magnitude exceeding $2^{-(i+1)}$.
The pseudo-code is provided in Algorithm \ref{algo:main_algo_dimensions_unknown}, where, $\forall i \geq 0$, $S_i$ in lines $15$ and $16$ is the total number of random-exploration samples in all phases upto and including $i$.

\begin{algorithm}[t!]
  \caption{Adaptive Linear Bandit (Dimension)}
  \begin{algorithmic}[1]
 \STATE  \textbf{Input:} Initial Phase length $T_0$ and slack $\delta > 0$.
 \STATE $\widehat{\theta}_0 = \mathbf{1}$, $T_{-1}=0$
 \FOR {Each epoch $i \in \{0,1,2,\cdots\}$}
 \STATE $T_i = 25^{i} T_0$, $\quad$  $\varepsilon_i \gets \frac{1}{2^{i}}$, $\quad$  $\delta_i \gets \frac{\delta}{2^{i}}$
 \STATE $\mathcal{D}_i := \{i : |\widehat{\theta}_i| \geq \frac{\varepsilon_i}{2} \}$
 \FOR {Times $t \in \{T_{i-1}+1,\cdots,T_i\}$}
 \STATE Play $\text{OFUL}(1,\delta_i)$ only restricted to coordinates in $\mathcal{D}_i$. Here $\delta_i$ is the probability slack parameter and $1$ represents $\|\theta^*\| \leq 1$.
 \ENDFOR
 \FOR {Times $t \in \{T_i+1,\cdots,T_i + 5^i\sqrt{T_0}\}$}
 \STATE Play an arm from the action set $\mathcal{A}$ chosen uniformly and independently at random.
 \ENDFOR
 \STATE $\boldsymbol{\alpha}_i \in \real^{S_i \times d}$ with each row being  the arm played during all random explorations in the past.
 \STATE $\boldsymbol{y}_i \in \real^{S_i}$  with $i$-th entry being the observed reward at the $i$-th random exploration in the past
 \STATE $\widehat{\theta}_{i+1} \gets (\boldsymbol{\alpha}_i^T\boldsymbol{\alpha}_i)^{-1}\boldsymbol{\alpha}_i\mathbf{y}_i$, is a $d$ dimensional vector
 \ENDFOR
  \end{algorithmic}
  \label{algo:main_algo_dimensions_unknown}
\end{algorithm}

\subsection{Main Result}

We first specify, how to set the input parameter $T_0$, as function of $\delta$.
For any $N \geq d$, denote by $A_N$ to be the $N \times d$ random matrix with each row being a vector sampled uniformly and independently from the unit sphere in $d$ dimensions.
Denote by $M_N := \frac{1}{N} \mathbb{E}[A_N^TA_N]$, and by $\lambda_{max}^{(N)},\lambda_{min}^{(N)}$, to be the largest and smallest eigenvalues of $M_N$. Observe that as $M_N$ is positive semi-definite ($0 \leq \lambda_{min}^{(N)}\leq\lambda_{max}^{(N)}$) and almost-surely full rank, i.e., $\mathbb{P}[\lambda_{min}^{(N)} > 0] = 1$.
The constant $T_0$ is the smallest integer such that
\begin{equation}
    \sqrt{T_0} \geq \max \left ( \frac{32\sigma^2}{(\lambda_{min}^{(\lceil \sqrt{T_0} \rceil)})^2}\ln (2d/\delta), \frac{4}{3} \frac{(6\lambda_{max}^{(\lceil \sqrt{T_0}\rceil)}+\lambda_{min}^{(\lceil \sqrt{T_0}\rceil)})(d+\lambda_{max}^{(\lceil \sqrt{T_0}\rceil)})}{(\lambda_{min}^{(\lceil \sqrt{T_0}\rceil)})^2}\ln ( 2d/\delta) \right )
    \label{eqn:T_0_defn}
\end{equation}

\begin{remark}
$T_0$ in Equation (\ref{eqn:T_0_defn}) is chosen such that, at the end of phase $0$, 
$
    \mathbb{P}[||\widehat{\theta}_0 - \theta^*||_{\infty} \geq 1/2 ] \leq \delta.
$
\end{remark}
A formal statement of the Remark is provided in Lemma \ref{lem:reg_bounds} in Appendix \ref{sec:proofs}.

\begin{theorem}
Suppose Algorithm \ref{algo:main_algo_dimensions_unknown} is run with input parameters $\delta \in (0,1)$, and $T_0$ as given in Equation (\ref{eqn:T_0_defn}), then with probability at-least $1-\delta$, the regret after a total of $T$ arm-pulls satisfies
\begin{align*}
R_T &\leq \frac{50}{{\gamma^{4.65}}}T_0
+ 25 \sqrt{T} [ 1 +  4 \sqrt{d^*\ln ( 1 + \frac{25T}{d^*} )} (1 + \sigma\sqrt{2 \ln ( \frac{T}{T_0\delta} ) + d^* \ln ( 1+\frac{25T}{d^*})})].
\end{align*}
The parameter $\gamma > 0$ is the minimum  magnitude of the non-zero coordinate of $\theta^*$, i.e., $\gamma = \min \{|\theta^*_i| : \theta^*_i \neq 0 \}$ and $d^*$ the sparsity of $\theta^*$, i.e., $d^* = |\{i:\theta^*_i\neq 0\}|$.
\label{thm:adaptive_dimension}
\end{theorem}

In order to parse this result, we give the following corollary.
\begin{corollary}
Suppose Algorithm \ref{algo:main_algo_dimensions_unknown} is run with input parameters $\delta \in (0,1)$, and $T_0 = \widetilde{O} \left(d^2\ln^2 \left( \frac{1}{\delta} \right) \right)$ given in Equation (\ref{eqn:T_0_defn}), then with probability at-least $1-\delta$, the regret after $T$ times satisfies
\begin{align*}
        R_T &\leq O ( \frac{d^2}{{\gamma^{4.65}}} \ln^2 ( d/\delta) )  + \widetilde{O} ( d^* \sqrt{ T}).
\end{align*}
\label{cor:dimension_adaptation}
\end{corollary}

\begin{remark}
The constants in the Theorem are not optimized. In particular,
the exponent of $\gamma$  can be made arbitrarily close to $4$, by setting $\varepsilon_i = C^{-i}$ in Line $4$ of Algorithm \ref{algo:main_algo_dimensions_unknown}, for some appropriately large constant $C > 1$, and increasing $T_i = (C')^iT_0$, for appropriately large $C'$ ($C'\approx C^4)$.
\end{remark}

\noindent {\textbf{Discussion - }}
The regret of an oracle algorithm that knows the true complexity $d^*$ scales as $\widetilde{O}(d^*\sqrt{T})$ \cite{sparse_bandit1,sparse_bandit2}, matching  {\ttfamily ALB-Dim}'s regret, upto an additive constant independent of time.
{\ttfamily ALB-Dim} is the first algorithm to achieve such model selection guarantees.
On the other hand, standard linear bandit algorithms such as {\ttfamily OFUL} achieve a regret scaling $\widetilde{O}(d\sqrt{T})$, which is much larger compared to that of {\ttfamily ALB-Dim}, especially when $d^* << d$, and $\gamma$ is a constant.
Numerical simulations further confirms this deduction, thereby indicating that our improvements are fundamental and not from mathematical bounds.
 Corollary \ref{cor:dimension_adaptation} also indicates that {\ttfamily ALB-Dim} has higher regret if $\gamma$ is lower. A small value of $\gamma$ makes it harder to distinguish a non-zero coordinate from a zero coordinate, which is reflected in the regret scaling. 
Nevertheless, this only affects the \emph{second order term as a constant}, and the dominant scaling term only depends on the true complexity $d^*$, and not on the underlying dimension $d$.
However, the regret guarantee is not uniform over all $\theta^*$ as it depends on $\gamma$. Obtaining regret rates matching the oracles and that hold uniformly over all $\theta^*$ is an interesting avenue of future work.

\section{Dimension as a Measure of Complexity - Finite Armed Setting}

\subsection{Problem Setup} 
In this section, we consider the model selection problem for the setting with finitely many arms in the framework studied in \cite{foster_model_selection}. 
At each time $t \in [T]$, the forecaster is shown a context $X_t \in \mathcal{X}$, where $\mathcal{X}$ is some arbitrary `feature space'. The set of contexts $(X_t)_{t=1}^T$ are i.i.d. with $X_t \sim \mathcal{D}$, a probability distribution over $\mathcal{X}$ that is known to the forecaster.
Subsequently, the forecaster chooses an action $A_t \in \mathcal{A}$, where the set $\mathcal{A} \coloneqq \{1,\cdots,K\}$ are the $K$ possible actions chosen by the forecaster. The forecaster then receives a reward $Y_t := \langle \theta^*, \phi^M(X_t,A_t) \rangle + \eta_t$. Here $(\eta_t)_{t =1}^T$ is an i.i.d. sequence of $0$ mean sub-gaussian random variables with sub-gaussian parameter $\sigma^2$ that is known to the forecaster.
The function\footnote{Superscript $M$ will become clear shortly} $\phi^M : \mathcal{X} \times \mathcal{A} \rightarrow \mathbb{R}^d$ is a known feature map, and $\theta^* \in \mathbb{R}^d$ is an unknown vector. 
The goal of the forecaster is to minimize its regret, namely $R(T) \coloneqq \sum_{t=1}^T \mathbb{E}\left[ \langle A^*_t - A_t,\theta^* \rangle\right]$, where at any time $t$, conditional on the context $X_t$, $A^*_t \in \argmax_{a \in \mathcal{A}} \langle a,\phi^M(X_t,a) \rangle$. Thus, $A^*_t$ is a random variable as $X_t$ is random.

To describe the model selection, we consider a sequence of $M$ dimensions $1 \leq d_1 < d_2,\cdots < d_M \coloneqq d$ and an associated set of feature maps $(\phi^m)_{m=1}^M$, where for any $m \in [M]$, $\phi^m(\cdot,\cdot) : \mathcal{X} \times \mathcal{A} \rightarrow \mathbb{R}^{d_i}$, is a feature map embedding into $d_i$ dimensions. Moreover, these feature maps are nested, namely, for all $m \in [M-1]$, for all $x \in \mathcal{X}$ and $a \in \mathcal{A}$, the first $d_m$ coordinates of $\phi^{m+1}(x,a)$ equals $\phi^m(x,a)$. The forecaster is assumed to have knowledge of these feature maps.
The unknown vector $\theta^*$ is such that its first $d_{m^*}$ coordinates are non-zero, while the rest are $0$. 
The forecaster does no know the true dimension $d_{m^*}$.
Thus, although, the dimensionality of the problem is $d_{m^*}$, which is unknown to the forecaster. If this were known, than standard contextual bandit algorithms such as LinUCB \cite{chu2011contextual} can guarantee a regret scaling as $\widetilde{O}(\sqrt{d_{m^*}T})$.  In this section, we provide an algorithm in which, even when the forecaster is unaware of $d_{m^*}$, the regret scales as $\widetilde{O}(\sqrt{d_{m^*}T})$. However, this result is non uniform over all $\theta^*$ as, we will show, depends on the minimum non-zero coordinate value in $\theta^*$.
\vspace{2mm}

\noindent\textbf{Model Assumptions} We will require some assumptions identical to the ones stated in \cite{foster_model_selection}. Let $\|\theta^*\|_2 \leq 1$, which is known to the forecaster. The distribution $\mathcal{D}$  is assumed to be known to the forecaster. Associated with  the distribution $\mathcal{D}$
 is a matrix $\Sigma_M \coloneqq \frac{1}{K} \sum_{a \in \mathcal{A}} \mathbb{E} \left[ \phi^M(x,a) \phi^M(x,a)^T\right]$ (where $x \sim \mathcal{D}$), where we assume its minimum eigen value $\lambda_{min}(\Sigma_M)  > 0$ is strictly positive.
Further, we assume that, for all $a \in \mathcal{A}$, the random variable $\phi^M(x,a)$ (where $x \sim \mathcal{D}$ is random) is a sub-gaussian random variable with (known) parameter $\tau^2$.

\subsection{{\ttfamily ALB-Dim} Algorithm}

The algorithm in this case is identical to that of Algorithm \ref{algo:main_algo_dimensions_unknown}, except with the difference that in place of OFUL, we use {\ttfamily SupLinRel} of \cite{linRel} as the black-box. 
The full details of the Algorithm are provided in Appendix \ref{appendix-comparision}.

\subsection{Main Result}

For brevity, we only state the Corollary of our main Theorem (Theorem \ref{thm:adaptive_dimension_foster}) which is stated in Appendix \ref{appendix-comparision}.

\begin{corollary}
Suppose Algorithm \ref{algo:main_algo_dimensions_foster} is run with input parameters $\delta \in (0,1)$, and $T_0 = \widetilde{O} \left(d^2\ln^2 \left( \frac{1}{\delta} \right) \right)$ given in Equation (\ref{eqn:T_0_defn_foster}) , then with probability at-least $1-\delta$, the regret after $T$ times satisfies
\begin{align*}
        R_T &\leq O \left( \frac{d^2}{{\gamma^{4.65}}} \ln^2 ( d/\delta) \tau^2 \ln \left( \frac{TK}{\delta}\right) \right)  + \widetilde{O} (  \sqrt{ T d^*_m}),
\end{align*}
where $\gamma = \min \{|\theta^*_i| : \theta^*_i \neq 0 \}$ and $d^*$ the sparsity of $\theta^*$.
\label{cor:dimension_adaptation_foster}
\end{corollary}

\noindent {\textbf{Discussion - }}
Our regret scaling (in time) matches that of an oracle that knows the true problem complexity and thus obtains a regret scaling of $\widetilde{O}(\sqrt{d_{m^*}T})$. This, thus improves on the rate compared to that obtained in \cite{foster_model_selection}, whose regret scaling is sub-optimal compared to the oracle. On the other hand however, our regret bound depends on $\gamma$ and is thus not uniform over all $\theta^*$, unlike the bound in \cite{foster_model_selection} that is uniform over $\theta^*$. Thus, in general, our results are not directly comparable to that of \cite{foster_model_selection}. It is an interesting future work to close the gap and in particular, obtain the regret matching that of an oracle to hold uniformly over all $\theta^*$.

\section{Simulations}
\label{sec:simulations}

\begin{figure*}[t!]
    \centering
    \subfigure[$\|\theta^*\| = 0.1, \,\, b_1 =10$]{\includegraphics[height = 3.65cm,width=3.3cm]{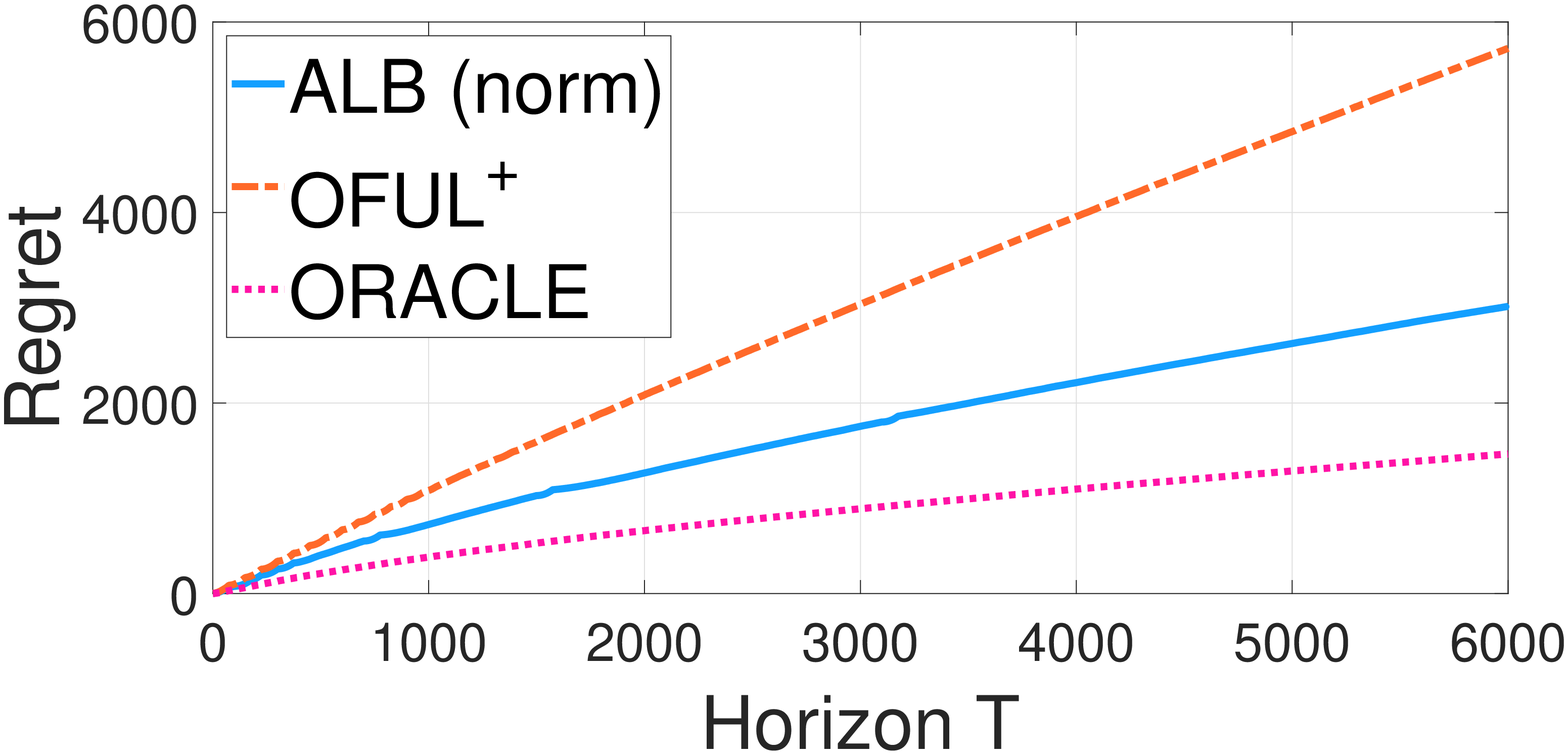} }
    \subfigure[$\|\theta^*\| = 1, \,\, b_1 =10$]{\includegraphics[height = 3.65cm,width=3.3cm]{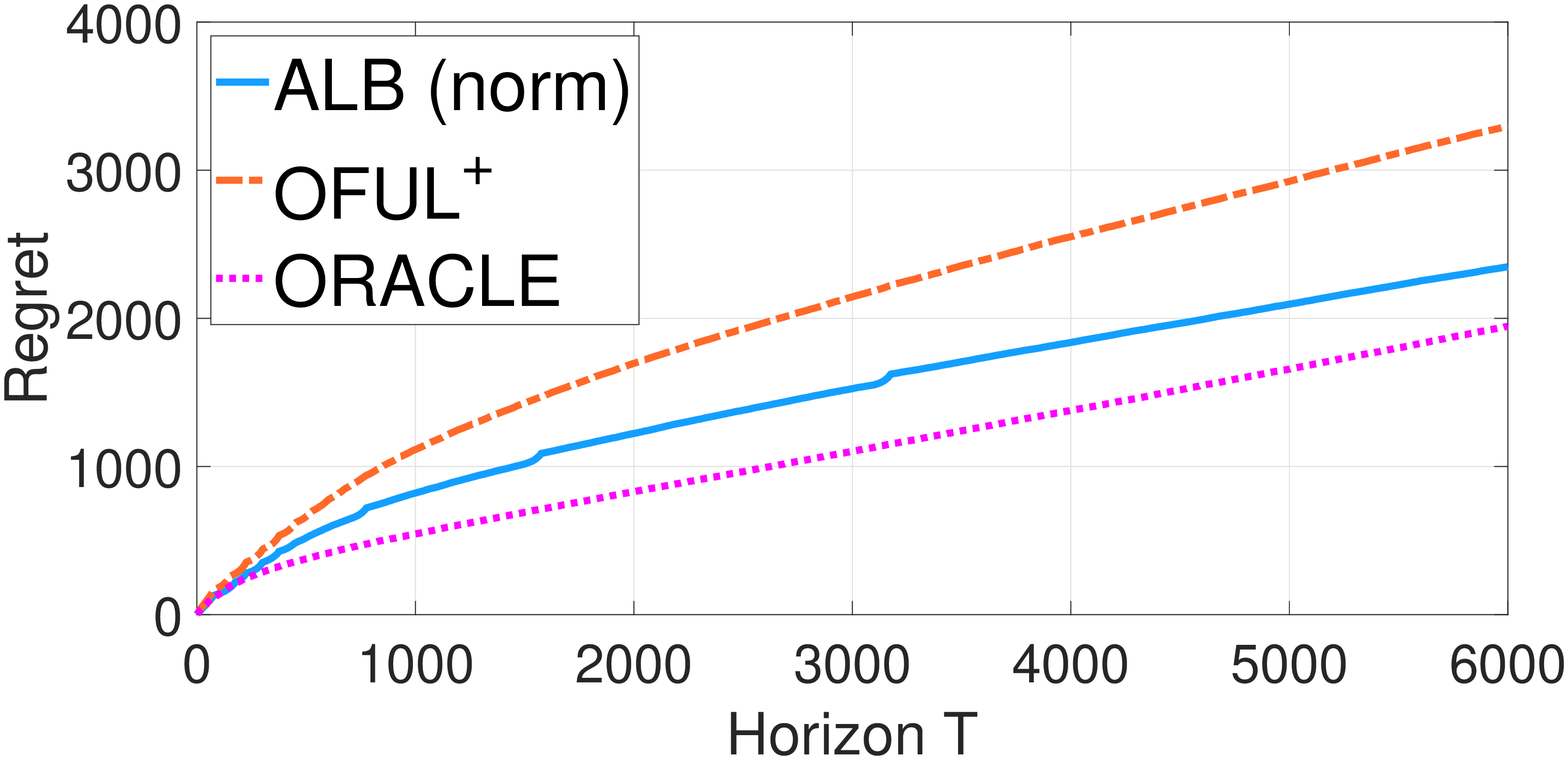} }
    \subfigure[Estimates of $\|\theta^*\|$ ]{{\includegraphics[height = 3.65cm,width=3.3cm]{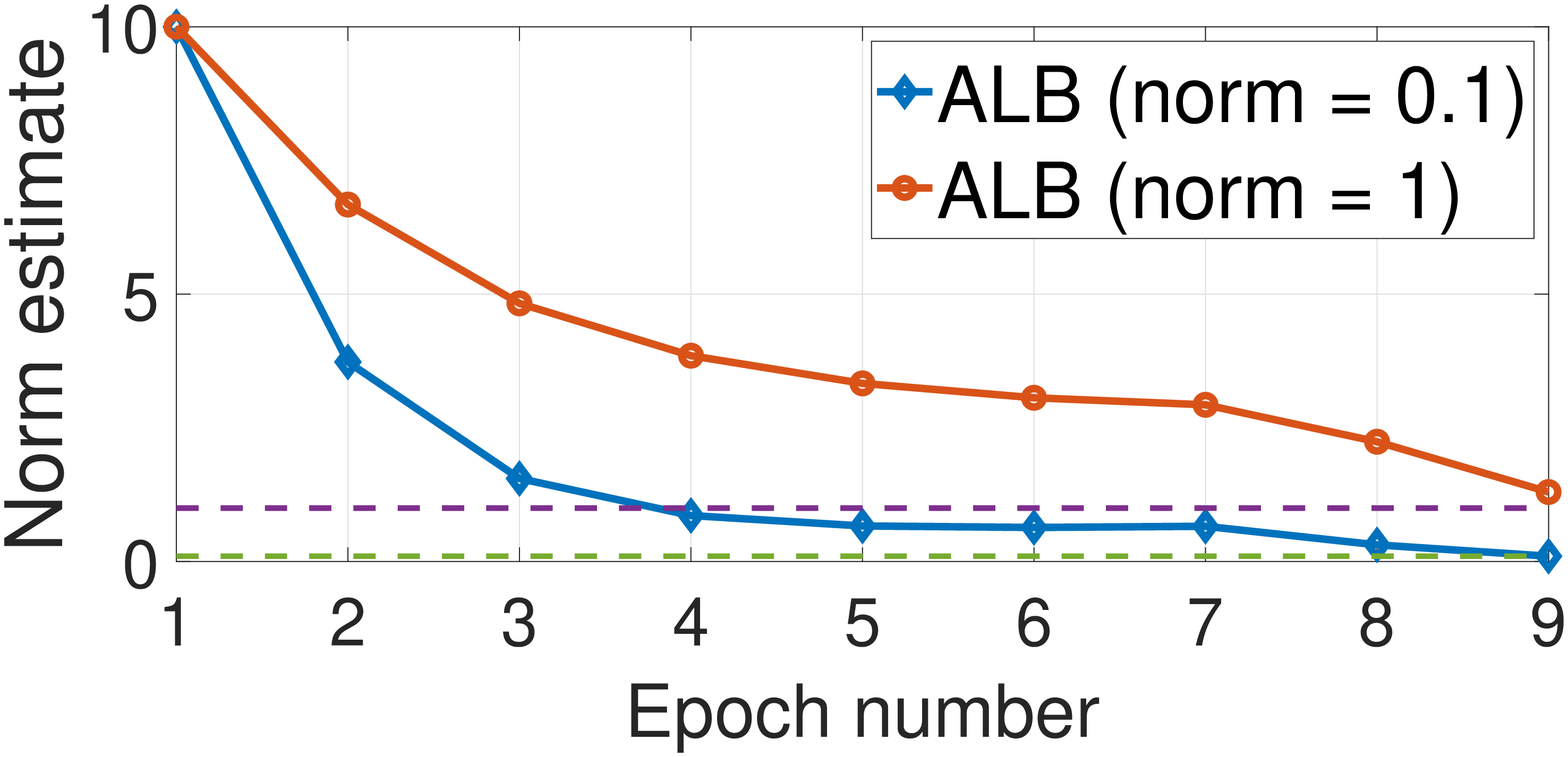} }}
    \subfigure[$d^*=20$, $d=500$ ]{{\includegraphics[height = 3.65cm,width=3.3cm]{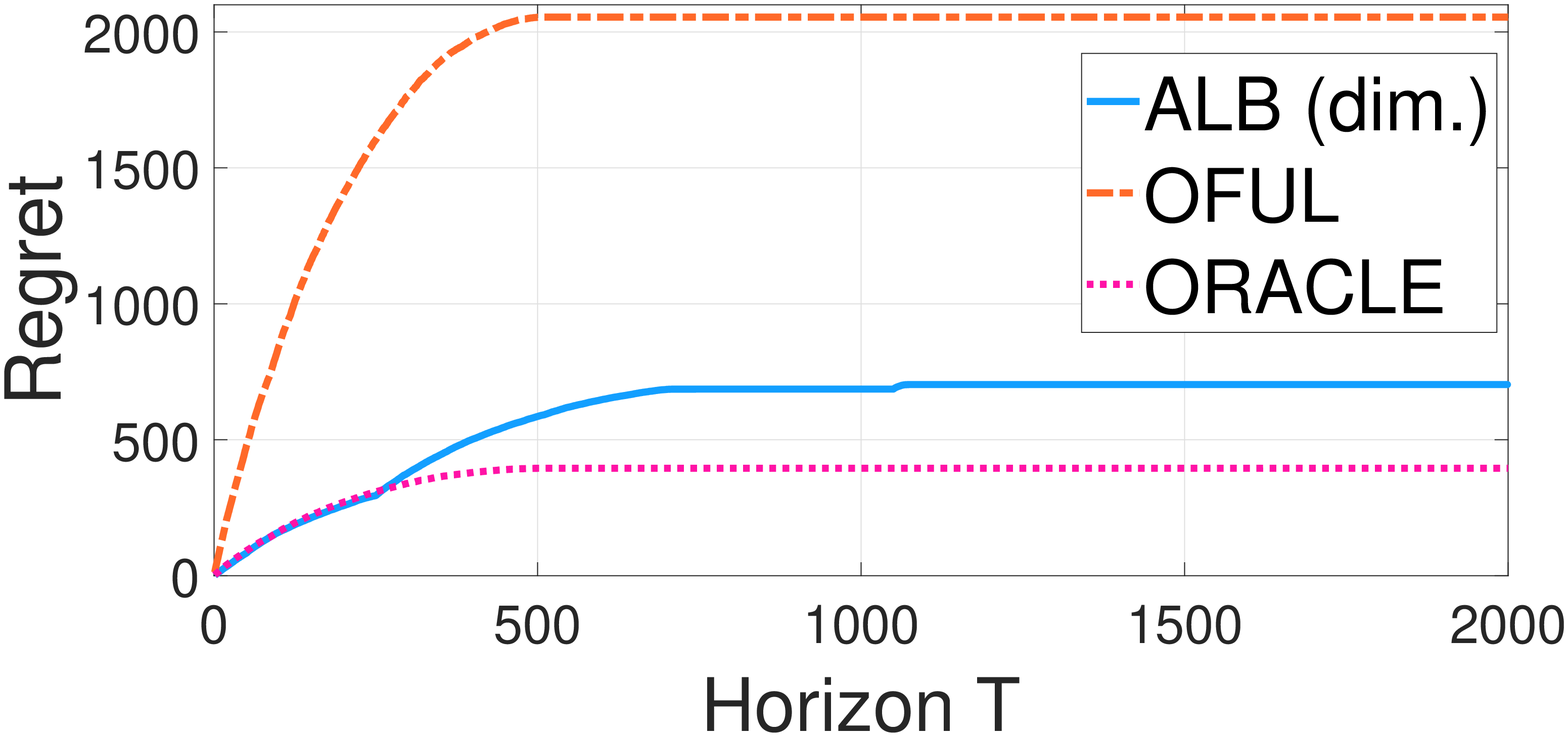} }}
    \subfigure[$d^*=20$, $d=200$ ]{{\includegraphics[height = 3.65cm,width=3.3cm]{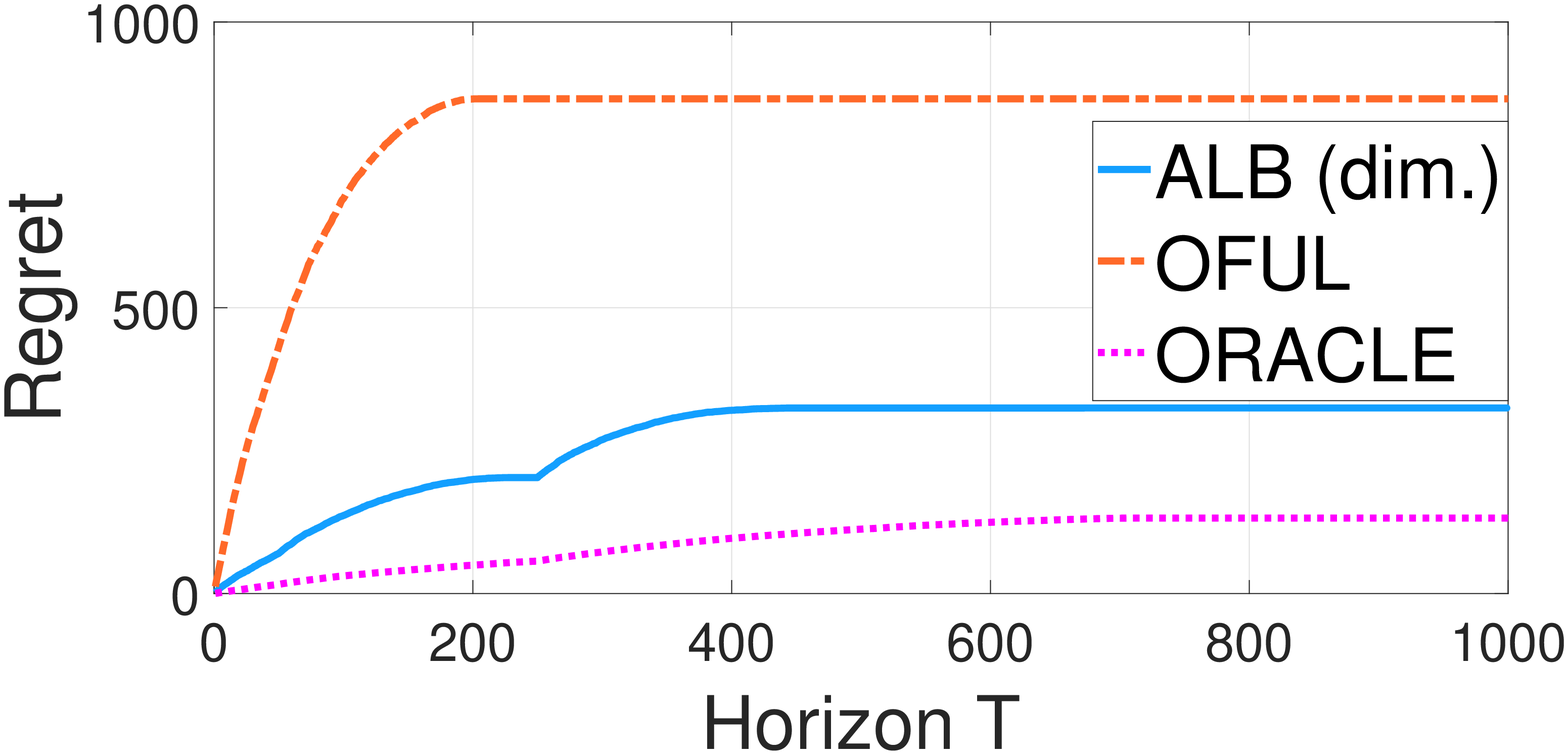} }}
    \subfigure[Dimension refinement ]{{\includegraphics[height = 3.65cm,width=3.3cm]{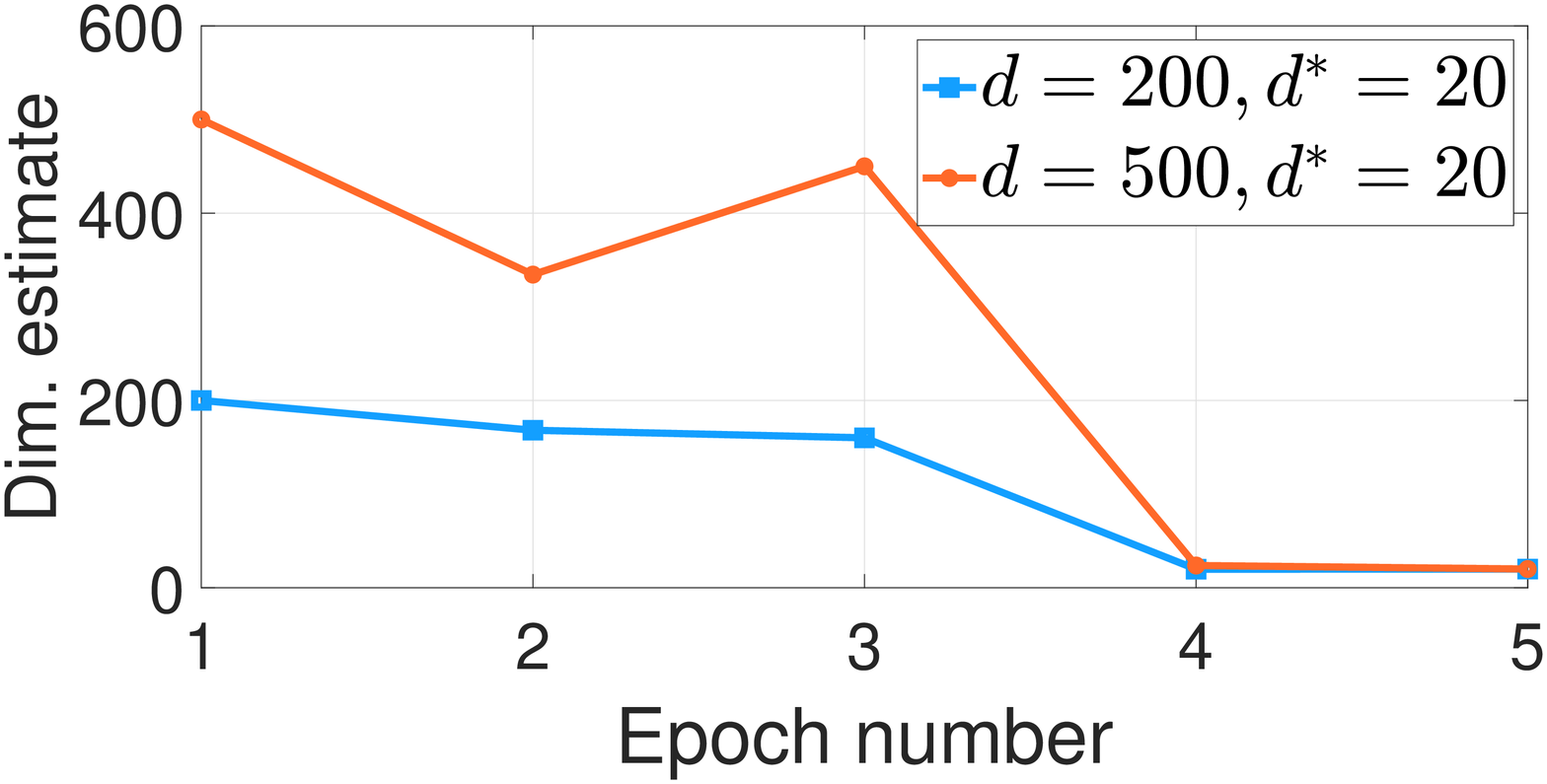} }}
    \subfigure[Yahoo; $ b_1 = 25$]{\includegraphics[height = 3.65cm,width=3.3cm]{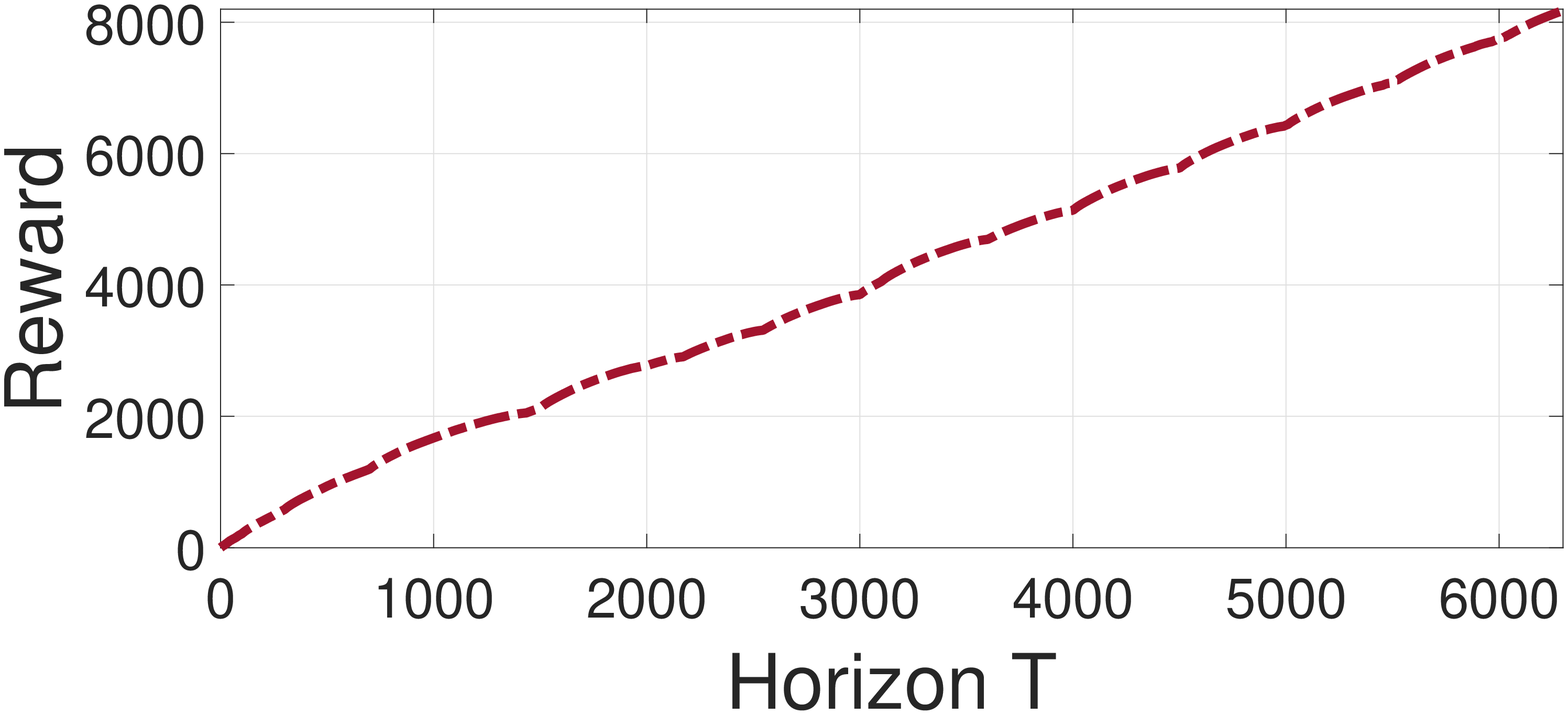} }
    \subfigure[Yahoo; $ b_1 = 25 $]{\includegraphics[height = 3.65cm,width=3.3cm]{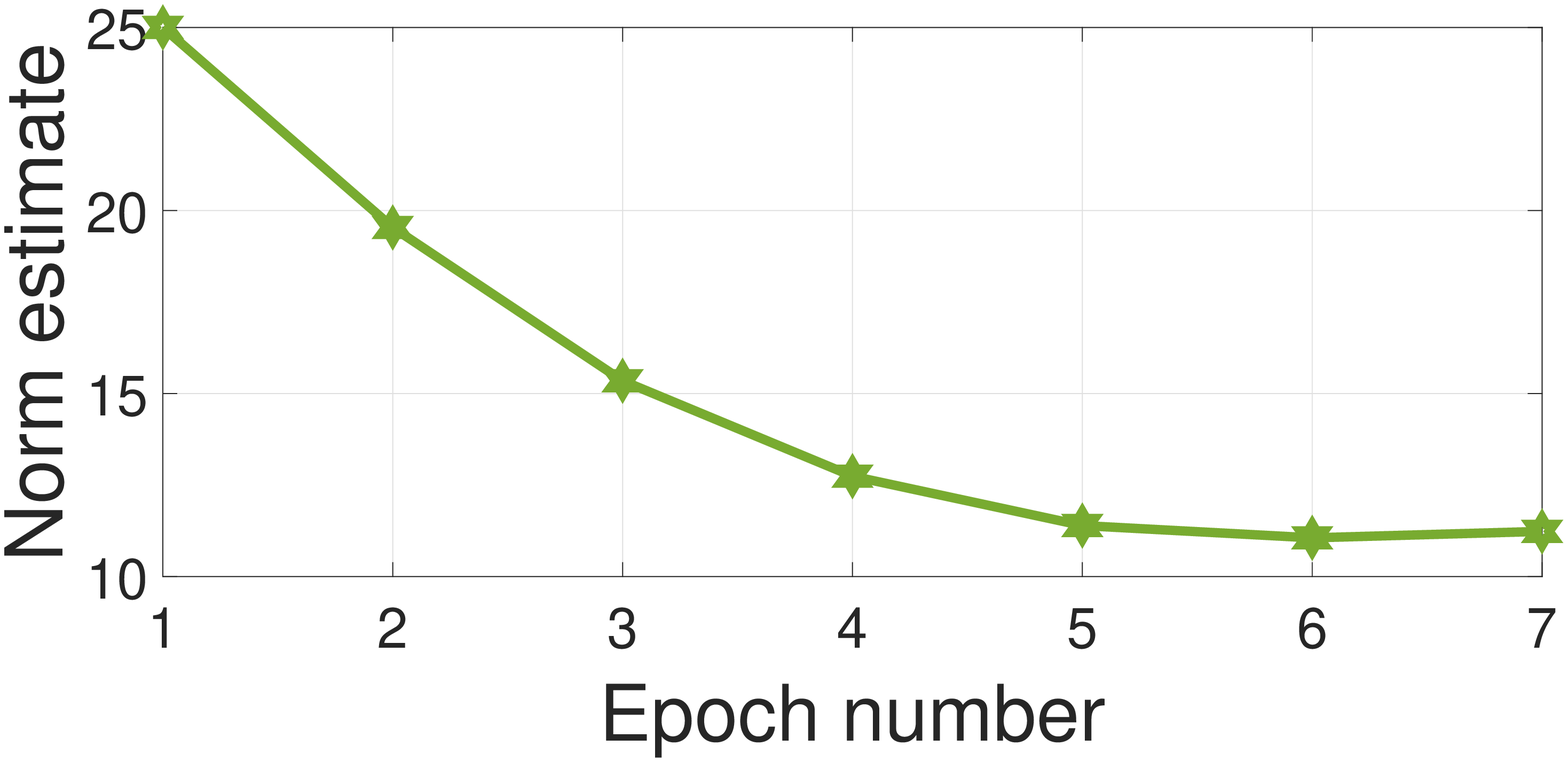} }
    \caption{Synthetic and real-data experiments, validating the effectiveness of Algorithm~\ref{algo:main_algo} and \ref{algo:main_algo_dimensions_unknown}. All the results are averaged over $25$ trials. }%
    \label{fig:experiments}
\end{figure*}
\subsection{Synthetic Experiments}

We compare {\ttfamily ALB-Norm}
with the (non-adaptive) OFUL$^+$ and an \emph{oracle} that knows the problem complexity apriori. The oracle just runs OFUL$^+$ with the known problem complexity. We choose the bias $\sim \mathcal{U}[-1,1]$, and the additive noise to be zero-mean Gaussian random variable with variance $0.5$. At each round of the learning algorithm, we sample the context vectors from a $d$-dimensional standard Gaussian, $\mathcal{N}(0,I_d)$. We select $d = 50$, the number of arms, $K=75$, and the initial epoch length as $100$. In particular, we generate the true $\theta^*$ in $2$ different ways: (i) $\|\theta^*\| = 0.1$, but the initial estimate $b_1 =10$, and (ii) $\|\theta^*\| = 1$, with the initial estimate $b_1 =10$.

In panel (a) and (b) of Figure~\ref{fig:experiments}, , we observe that, in setting (i), OFUL$^+$ performs poorly owing to the gap between $\|\theta^*\|$ and $b_1$. On the other hand, {\ttfamily ALB-Norm} is sandwiched between the OFUL$^+$ and the oracle. Similar things happen in setting (ii). In panel (c), we show that the norm estimates of {\ttfamily ALB-Norm} improves over epochs, and converges to the true norm very quickly.

In panel (d)-(f), we compare the performance of {\ttfamily ALB-Dim} with the OFUL (\cite{oful}) algorithm and an \emph{oracle} who knows the true support of $\theta^*$ apriori. 
For computational ease, we set $\varepsilon_i = 2^{-i}$ in simulations.
We select $\theta^*$ to be $d^* =20$-sparse, with the smallest non-zero component, $\gamma = 0.12$. We have $2$ settings: (i) $d = 500$ and (ii) $d = 200$. In panel (d) and (e), we observe a huge gap in cumulative regret between {\ttfamily ALB-Dim} and OFUL, thus showing the effectiveness of dimension adaptation. In panel (f), we plot the successive dimension refinement over epochs. We observe that within $4-5$ epochs, {\ttfamily ALB-Dim} finds the sparsity of $\theta^*$.

\subsection{Real-data experiment}

Here, we evaluate the performance of {\ttfamily ALB-Norm} on Yahoo!     `Learning to Rank Challenge' dataset (\cite{yahoo}). In particular, we use the file \texttt{set2.test.txt}, which consists of $103174$ rows and $702$ columns. The first column denotes the rating, $\{0,1,.,4\}$ given by the user (which is taken as reward); the second column denotes the user id, and the rest $700$ columns denote the context of the user. After selecting $20,000$ rows and $50$ columns at random (several other random selections yield similar results), we cluster the data by running $k$ means algorithm with $k=500$. We treat each cluster as a bandit arm with mean reward as the empirical mean of the individual rating in the cluster, and the context as the centroid of the cluster. This way, we obtain a bandit setting with $K = 500$ and $d =50$. 

Assuming (reward, context) coming from a linear  model (with bias, see Section~\ref{sec:setup}), we use {\ttfamily ALB-Norm} to estimate the bias and $\theta^*$ simultaneously. In panel (g), we plot the cumulative reward accumulated over time. We observe that the reward is accumulated over time in an almost linear fashion. We also plot the norm estimate, $\|\theta^*\|$ over epochs in panel (h), starting with an initial estimate of $25$. We observe that within $6$ epochs the estimate stabilizes to a value of $11.1$. This shows that {\ttfamily ALB-Norm} adapts to the actual $\|\theta^*\|$.

\section{Conclusion}
In this paper, we considered refined model selection for linear bandits, by defining new notions of complexity.
We gave two novel algorithms {\ttfamily ALB-Norm} and {\ttfamily ALB-Dim} that successively refines the hypothesis class and achieves model selection guarantees; regret scaling in the complexity of the smallest class containing the true model. This is the first such algorithm to achieve regret scaling similar to an oracle that knew the problem complexity. 
 {\color{black} An interesting direction of future work is to derive regret bounds for the case when the dimension is a measure of complexity, that hold uniformly over all $\theta^*$, i.e., have no explicit dependence on $\gamma$}.

\section{Acknowledgements} 
The authors would like to acknowledge Akshay Krishnamurthy, Dylan Foster and Haipeng Luo for insightful comments and suggestions.

\bibliographystyle{alpha}
\bibliography{ref-adaptive-bandits}

\appendix

\section*{Appendix}

\section{Detailed Description of OFUL$^+$}
\label{app:oful}
We now discuss the algorithm OFUL$^+$. A variation of this was proposed in \cite{osom} in the context of model selection between linear and standard multi-armed bandits. As seen in the OFUL$^+$ sub-routine of Algorithm~\ref{algo:main_algo}, we use $\Tilde{\mu}_{i,t}$ to address the bias term in the observation, which we define shortly. The parameters $b$ and $\delta$ appears in the construction of the confidence set and the regret guarantee. Furthermore, assume that the algorithm OFUL$^+$ is run for $\Tilde{T}$ rounds.

Let $A_s$ be the arm index played at time instant $s$ and $T_i(t)$ be the number of times we play arm $i$ until time $t$. Hence $T_i(t) = \sum_{s=1}^t \ind{A_s = i}$. Also, let $b$ be the current estimate of $\|\theta^*\|$. Also define,
\begin{align*}
    \Bar{g}_{i,t} = \frac{1}{T_i(t)}\sum_{s=1}^t g_{i,s} \ind{A_s =t}.
\end{align*}
With this, we have
\begin{align}
    \Tilde{\mu}_{i,t} = \Bar{g}_{i,t} + \sigma \left [ \frac{1+T_i(t)}{T_i^2(t)} \left( 1+2\log \left(\frac{K(1+T_i(t))^{1/2}}{\delta} \right) \right)  \right]^{1/2} + b \, \sqrt{\frac{2d}{T_i(t)}\log \left(\frac{1}{\delta}\right)}
    \label{eqn:mu_tilde}
\end{align}

In order to specify the confidence interval $\mathcal{C}_t$, we first talk about the least squares estimate $\hat{\theta}$ first. Using the notation of \cite{osom}, we define
\begin{align*}
    \hat{\theta}_t = \left( \mathbf{\alpha}^\top_{K+1:t} \mathbf{\alpha}_{K+1:t} + I \right)^{-1} \mathbf{\alpha}^\top_{K+1:t} G_{K+1:t}
\end{align*}
where $\alpha_{K+1:t}$ is a matrix with rows $\alpha_{A_{K+1},K+1}^\top, \ldots, \alpha_{A_t,t}^\top$ and $G_{K+1:t} = [g_{A_{K+1},K+1} - \Tilde{\mu}_{A_{K+1},K+1}, \ldots, g_{A_t,t} - \Tilde{\mu}_{A_t,t}]^\top$. With this, the confidence interval is defined as
\begin{align}
    \mathcal{C}_t = \left \lbrace \theta \in \real^d: \|\theta - \hat{\theta}_t\| \leq \mathcal{K}_\delta(b,t,\Tilde{T})  \right \rbrace,
    \label{eqn:conf_int}
\end{align}
and Lemma 2 of \cite{osom} shows that $\theta^* \in \mathcal{C}_t$ with probability at least $1-4\delta$.

We now define  the quantity $\mathcal{K}_\delta(b,t,\Tilde{T})$. Note that we track the dependence on the complexity parameter $\|\theta^*\|$. We have
\begin{align}
    & T_{\min}(\delta,\Tilde{T}) = \left( \frac{16}{\rho^2_{\min}} + \frac{8}{3\rho_{\min}} \right) \log \left( \frac{2d\Tilde{T}}{\delta} \right), \nonumber \\ & \mathcal{M}_\delta(b,t) = b + \sqrt{2\sigma^2 \left(\frac{d}{2}\log \left(1+\frac{t}{d}\right) + \log\left(\frac{1}{\delta}\right) \right)},  \label{eqn:tmin}\\
    & \Upsilon_\delta(b,t,\Tilde{T}) = \frac{10}{3}\left( b +2 + \sigma\sqrt{1+2\log \left(\frac{2K\Tilde{T}}{\delta} \right)}\right) \nonumber \\
    & \times \left[ \log \left(\frac{2K\Tilde{T}}{\delta} \right) + \sqrt{t \log \left(\frac{2K\Tilde{T}}{\delta} \right)+  \log^2 \left(\frac{2K\Tilde{T}}{\delta} \right)} \right], \label{eqn:upsilon}\\
    & \mathcal{K}_\delta(b,t,\Tilde{T}) = \begin{cases} \mathcal{M}_\delta(b,t) + \Upsilon_\delta(b,t,\Tilde{T}) &\mbox{if } 1 < t < T_{\min}, \\
    \frac{\mathcal{M}_\delta(b,t)}{\sqrt{1+\rho_{\min} \, t/2}} + \frac{\Upsilon_\delta(b,t,\Tilde{T})}{1+\rho_{\min}\, t/2} & \mbox{if } t >  T_{\min}. \end{cases} \label{eqn:scriptk}
\end{align}

\section{Proofs of the main results}
\label{sec:proofs}
In this section, we collect the proof of our main results. We start with the norm-based complexity measure.

\subsection{Proof of Theorem~\ref{thm:regret_norm_based}}
We first take Lemma~\ref{lem:b_seq} for granted and conclude the proof of Theorem~\ref{thm:regret_norm_based} using the lemma. Suppose we play Algorithm~\ref{algo:main_algo} for $N$ epochs. The cumulative regret is given by
\begin{align*}
    R(T) \leq C_1 (2\tau + K) \|\theta^*\| + \sum_{i=1}^N R(\delta_i,b_i) (T_{i}),
\end{align*}
where $R(\delta_i,b_i) (T_{i})$ is the cumulative regret of the OFUL$^+_{\delta_i}(b_i)$ in the $i$-th epoch. As seen (by tracking the dependence on $\|\theta^*\|$) in \cite{osom}, the cumulative regret of OFUL$^+_{\delta_i}(b_i)$ scales linearly with $b_i$. Hence, we obtain
\begin{align*}
    R(T) \leq \sum_{i=1}^N b_i \, R(\delta_i,1)(T_{i}).
\end{align*}
Using Lemma 1, we obtain, with probability at least $1-8\delta_1$,
\begin{align*}
    R(T) \leq C_1 (2\tau + K) \|\theta^*\| + (c_1 \|\theta^*\| + c_2)  \sum_{i=1}^N  R(\delta_i,1)(T_{i})
\end{align*}
Theorem 3 of \cite{osom} gives,
\begin{align}
    R(\delta_i,1)(T_{i}) \leq  C (\sqrt{K} + \sqrt{d}) \sqrt{T_{i}} \log\left( \frac{K T_i}{\delta_i}\right)
\end{align}
with probability exceeding $1-5\delta_i$. With the doubling trick, we have
\begin{align*}
    T_i = 2^{i-1} T_1, \quad \quad \delta_i = \frac{\delta_1}{2^{i-1}}.
\end{align*}
Substituting, we obtain
\begin{align*}
     R(\delta_i,1)(T_{i}) \leq  C_1 (\sqrt{K} + \sqrt{d}) \sqrt{T_{i}} \log\left( \frac{K T_i}{\delta_i}\right) \left[ (2i -2) \log \left(\frac{K T_1}{\delta_1} \right) \right]
\end{align*}
with probability at least $1-5\delta_i$.

Using the above expression, we obtain
\begin{align*}
    R(T) \leq C_1 (2\tau + K) \|\theta^*\| + (C_2 \|\theta^*\| + C_3) \sum_{i=1}^N (\sqrt{K} + \sqrt{d}) \sqrt{T_{i}}  \left[ (2i -2) \log \left(\frac{K T_1}{\delta_1} \right) \right]
\end{align*}
with probability 
\begin{align*}
    & \geq 1- 8 \delta_1 - 5 \delta_1 \left( 1+ \frac{1}{2} + .. N\text{-th term} \right) \\
    & \geq 1- 8 \delta_1 - 5 \delta_1 \left( 1+ \frac{1}{2} + ... \right) \\
    & = 1- 8 \delta_1 - 10 \delta_1 \\
    & = 1-18\delta_1,
\end{align*}
where the term $8 \delta_1$ comes from Lemma~\ref{lem:b_seq}. Also, from the doubling principle, we obtain
\begin{align*}
    \sum_{i=1}^N 2^{i-1}T_1 = T \,\, \Rightarrow \, N  = \log_2 \left(1+ \frac{T}{T_1} \right).
\end{align*}
Using the above expression, we obtain
\begin{align*}
    R(T) & \leq C_1 (2\tau + K) \|\theta^*\| + (C_2 \|\theta^*\| + C_3) \sum_{i=1}^N (\sqrt{K} + \sqrt{d}) \sqrt{T_{i}}  \left[ (2i -2) \log \left(\frac{K T_1}{\delta_1} \right) \right] \\
    & \leq  C_1 (2\tau + K) \|\theta^*\| + 2(C_2 \|\theta^*\| + C_3) (\sqrt{K} + \sqrt{d}) \log \left(\frac{K T_1}{\delta_1} \right) \sum_{i=1}^N i \sqrt{T_i} \\
    & \leq  C_1 (2\tau + K) \|\theta^*\| + 2(C_2 \|\theta^*\| + C_3) (\sqrt{K} + \sqrt{d}) \log \left(\frac{K T_1}{\delta_1} \right) N \sum_{i=1}^N  \sqrt{T_i} \\
    & \leq  C_1 (2\tau + K) \|\theta^*\| + 2(C_2 \|\theta^*\| + C_3) (\sqrt{K} + \sqrt{d}) \log \left(\frac{K T_1}{\delta_1} \right) \log \left(\frac{T}{T_1}\right) \sum_{i=1}^N  \sqrt{T_i} \\
    & \leq  C_1 (2\tau + K) \|\theta^*\| + C( \|\theta^*\| + 1) (\sqrt{K} + \sqrt{d}) \log \left(\frac{K T_1}{\delta_1} \right) \log \left(\frac{T}{T_1}\right)  \sqrt{T},
\end{align*}
where the last inequality follows from the fact that
\begin{align*}
    \sum_{i=1}^N \sqrt{T_i} & = \sqrt{T_N}\left(1 + \frac{1}{\sqrt{2}} + \frac{1}{2} + .. N\text{-th term} \right)\\
    & \leq \sqrt{T_N}\left(1 + \frac{1}{\sqrt{2}} + \frac{1}{2} + ... \right)\\
    & = \frac{\sqrt{2}}{\sqrt{2} -1} \sqrt{T_N} \\
    & \leq \frac{\sqrt{2}}{\sqrt{2} -1} \sqrt{T}.
\end{align*}
The above regret bound holds with probability at least $1-18\delta_1$. 
 
 \vspace{3mm}
\subsection{Proof of Lemma~\ref{lem:b_seq}}
Let us consider the $i$-th epoch, and let $\hat{\theta}_{\mathcal{E}_i}$ be the least square estimate of $\theta^*$ at the end of epoch $i$. From the above section, the confidence interval at the end of epoch $i$, is given by
\begin{align*}
    \mathcal{C}_{\mathcal{E}_i} = \left \lbrace \theta \in \real^d: \|\theta - \hat{\theta}_{\mathcal{E}_i}\| \leq \mathcal{K}_{\delta_i}(b_i, T_i, T_i) \right \rbrace
\end{align*}
where we play OFUL$^+_{\delta_i}(b_i)$ during the $i$-th epoch, and $T_i$ is the number of total rounds in the $i$-th epoch. By choosing $T_1 > T_{\min}(\delta,T)$, we ensure that $T_i \geq T_{\min}(\delta,T_i)$. From equation~\eqref{eqn:scriptk}, and ignoring the non-dominant terms, we obtain
\begin{align*}
    \mathcal{K}_{\delta_i}(b_i,T_i,T_i) = \frac{\mathcal{M}_{\delta_i}(b_i,T_i)}{\sqrt{1+ \rho_{\min}\, T_i/2}} + \frac{\Upsilon_{\delta_i}(b_i,T_i,T_i)}{1+ \rho_{\min} \, T_i /2},
\end{align*}
with
\begin{align*}
    \mathcal{M}_{\delta_i}(b_i,T_i) \leq b_i + c_1 \sigma \sqrt{d} \log \left( \frac{T_i}{d \delta_i} \right)
\end{align*}
and
\begin{align*}
    \Upsilon_{\delta_i}(b_i,T_i,T_i) = 4 b_i \sqrt{T_i} \log \left(\frac{2KT_i}{\delta_i} \right) + c_2 \sigma \sqrt{T_i} \log  \left(\frac{2KT_i}{\delta_i} \right)
\end{align*}
Substituting the values, considering the dominating terms, and for a sufficiently large $T_i$, we obtain
\begin{align*}
    \mathcal{K}_{\delta_i}(b_i,T_i,T_i) & \leq \frac{7 b_i \log \left(\frac{2KT_i}{\delta_i}\right)}{\sqrt{\rho_{\min} \, T_i }} + C \frac{\sigma \sqrt{d}}{\sqrt{\rho_{\min} \,T_i}} \log \left(\frac{2KT_i}{\delta_i} \right) \\
    & \leq \frac{7 b_i \log \left(\frac{2KT}{\delta_i}\right)}{\sqrt{\rho_{\min} \,T_i}} + C \frac{\sigma \sqrt{d}}{\sqrt{\rho_{\min} \,T_i}} \log \left(\frac{2KT_i}{\delta_i} \right) 
\end{align*}
where $C$ is an universal constant.
From Lemma 2 of \cite{osom}, we know that $\theta^* \in \mathcal{C}_{\mathcal{E}_i}$ with probability at least $1-4\delta_i$. Hence, we obtain
\begin{align*}
    \|\hat{\theta}_{\mathcal{E}_i}\| \leq \|\theta^*\| + 2\mathcal{K}_{\delta_i}(b_i,T_i,T_i) \leq \|\theta^*\| + \frac{14 b_i \log \left(\frac{2KT_i}{\delta_i}\right)}{\sqrt{\rho_{\min}\, T_i}} + 2C \frac{\sigma \sqrt{d}}{\sqrt{\rho_{\min} \,T_i}} \log \left(\frac{2KT_i}{\delta_i} \right)
\end{align*}
Recall from Algorithm~\ref{algo:main_algo} that at the end of the $i$-th epoch, we set the length $T_{i+1} = 2 T_{i}$, and the estimate of $\|\theta^*\|$ is set to
\begin{align*}
    b_{i+1} = \max_{\theta \in \mathcal{C}_{\mathcal{E}_i}} \|\theta\|.
\end{align*}
From the definition of $\mathcal{C}_{\mathcal{E}_i}$, we obtain
\begin{align*}
    b_{i+1} =\|\hat{\theta}_{\mathcal{E}_i}\| + \mathcal{K}_{\delta_i}(b_i,T_i,T_i) \leq \frac{7 b_i \log \left(\frac{2KT_i}{\delta_i}\right)}{\sqrt{\rho_{\min} \,T_i}} + C \frac{\sigma \sqrt{d}}{\sqrt{\rho_{\min}\,T_i}} \log \left(\frac{2KT_i}{\delta_i} \right).
\end{align*}
Re-writing the above expression, with probability at least $1-4\delta_i$, we obtain
\begin{align}
    b_{i+1} &\leq \|\theta^*\| + \left(\frac{7 \log \left(\frac{2KT_i}{\delta_i}\right) }{\sqrt{\rho_{\min}}} \right) \frac{b_i }{\sqrt{T_i}} +  \left(\frac{ C \sigma \log \left(\frac{2KT_i}{\delta_i} \right)}{\sqrt{\rho_{\min}}} \right)\frac{ \sqrt{d}}{\sqrt{T_i}} \nonumber \\
    & \leq \|\theta^*\|+ i p \frac{b_i }{\sqrt{T_i}} + i q \frac{\sqrt{d}}{\sqrt{T_i}} \nonumber \\
    & \leq \|\theta^*\|+ i p \frac{b_i}{2^{\frac{i-1}{2}} \sqrt{T_1}} + i q \frac{\sqrt{d}}{2^{\frac{i-1}{2}}\sqrt{T_1}} \label{eqn:est_upper_bound}
\end{align}
where we use the fact that $\delta_i = \frac{\delta_1}{2^{i-1}}$ and $T_i = 2^{\frac{i-1}{2}}T_1$, and we have
\begin{align*}
    p = \left(\frac{14 \log \left(\frac{2KT_1}{\delta_1}\right) }{\sqrt{\rho_{\min}}} \right)
\end{align*}
and
\begin{align*}
    q = \left(\frac{ 2C \sigma \log \left(\frac{2KT_1}{\delta_1} \right)}{\sqrt{\rho_{\min}}} \right).
\end{align*}
Hence, we obtain
\begin{align*}
    b_{i+1} - b_i \leq \|\theta^*\| + i q \frac{\sqrt{d}}{2^{\frac{i-1}{2}}\sqrt{T_1}}  - \left( 1 - i p \frac{1}{2^{\frac{i-1}{2}}\sqrt{T_1}} \right) b_i.
\end{align*}
From the construction of $b_i$, we have $-b_i \leq -\|\theta^*\|$. Hence provided
\begin{align*}
    T_1 \geq  \frac{i^2 p^2}{2^{i-1}},
\end{align*}
which is equivalent to the condition $T_1 \geq 3 p^2$ (using the fact that $\frac{i^2}{2^{i-1}} \leq 3$ for $i \geq 1$), we obtain
\begin{align*}
    b_{i+1} - b_{i} \leq \left( i p \frac{1}{2^{\frac{i-1}{2}}\sqrt{T_1}} \right) \|\theta^*\| + i q \frac{\sqrt{d}}{2^{\frac{i-1}{2}}\sqrt{T_1}}.
\end{align*}
From the above expression, we obtain
\begin{align*}
    \sup_i b_i < \infty.
\end{align*}
with probability 
\begin{align*}
   & \geq 1- 4\delta_1 \left( 1+ \frac{1}{2} + \frac{1}{4}  + ...  \right) \\
   & = 1- 8 \delta_1.
\end{align*}
Invoking Equation~\eqref{eqn:est_upper_bound} and using the above fact in conjunction yield (with probability at least $1 - 8\delta_1$)
\begin{align*}
    \lim_{i \rightarrow \infty} b_i \leq \|\theta^*\|.
\end{align*}
However, from construction $b_i \geq \|\theta^*\|$. Using this, along with the above equation, we obtain
\begin{align*}
    \lim_{i \rightarrow \infty} b_i = \|\theta^*\|.
\end{align*}
with probability exceeding $1 - 8\delta_1$. So, the sequence $\{b_0,b_1,...\}$ converges to $\|\theta^*\|$ with high probability, and hence our successive refinement algorithm is consistent.
\vspace{2mm}
\paragraph{Rate of Convergence:} Since
\begin{align}
    b_i - b_{i-1} = \Tilde{\mathcal{O}}\left( \frac{i}{2^i} \right),
\end{align}
with probability greater than $1-4\delta_i$, the rate of convergence of the sequence $\{b_i\}_{i=0}^\infty$ is exponential in the number of epochs.

\vspace{2mm}
\paragraph{Uniform upper bound on $b_i$ for all $i$:} We now compute a uniform upper bound on $b_i$ for all $i$. Consider the sequence $ \Bigg \{\frac{i}{2^{\frac{i-1}{2}}} \Bigg \}_{i=1}^\infty$, and let $t_j$ denote the $j$-th term of the sequence. It is easy to check that $\sup_{i}t_i = 1.5$, and that the sequence $\{t_i\}_{i=1}^\infty$ is convergent. With this new notation, we have
\begin{align*}
    b_2 \leq \|\theta^*\| + t_1 \frac{p b_1}{\sqrt{T_1}} + t_1 \frac{q \sqrt{d}}{\sqrt{T_1}}.
\end{align*}
with probability exceeding $1-4\delta_1$. Similarly, for $b_3$, we have
\begin{align*}
    b_3 & \leq \|\theta^*\| + t_2 \frac{p b_1}{\sqrt{T_1}} + t_2 \frac{q \sqrt{d}}{\sqrt{T_1}} \\
    & \leq  \left (1+t_2 \frac{p}{\sqrt{T_1}} \right) \|\theta^*\| + \left( t_1 t_2 \frac{p}{\sqrt{T_1}} \frac{p}{\sqrt{T_1}} b_1 \right) + \left( t_1 t_2 \frac{p}{\sqrt{T_1}} \frac{q \sqrt{d}}{\sqrt{T_1}} + t_2 \frac{q \sqrt{d}}{\sqrt{T_1}}\right).
\end{align*}
with probability at least $1 - 4\delta_1 -4\delta_2 = 1-6\delta_1$. Similarly, we write expressions for $b_4,b_5,...$. Now, provided $T_1 \geq C_1 \left ( \max\{p,q\}  \, b_1 \right)^2 \, d$, where $C_1 > 9$ is a sufficiently large constant, the expression for $b_i$ can be upper-bounded as
\begin{align}
    b_i \leq \left(c_1 \| \theta^* \| + c_2\right),
\end{align}
with probability 
\begin{align*}
   & \geq 1- 4\delta_1 \left( 1+ \frac{1}{2} + \frac{1}{4}  + ... \text{upto } i\text{-th term} \right) \\
   & \geq 1- 4\delta_1 \left( 1+ \frac{1}{2} + \frac{1}{4}  + ...  \right) \\
   & = 1- 8 \delta_1.
\end{align*}
Here $c_1$ and $c_2$ are constants, and are obtained from summing an infinite geometric series with decaying step size. We also use the fact that $b_1 \geq 1$, and the fact that $\delta_i = \frac{\delta_1}{2^{i-1}}$.

\vspace{3mm}
\subsection{Proof of Theorem \ref{thm:adaptive_dimension}}

We shall need the following lemma from
\cite{linear_reg_guarantees}, on the behaviour of linear regression estimates.

\begin{lemma}
If $M \geq d$ and satisfies $M = O \left( \left(\frac{1}{\varepsilon^2} + d \right) \ln \left(\frac{1}{\delta}\right)\right)$, and $\widehat{\theta}^{(M)}$ is the least-squares estimate of $\theta^*$, using the $M$ random samples for feature, where each feature is chosen uniformly and independently on the unit sphere in $d$ dimensions, then with probability $1$, $\widehat{\theta}$ is well defined (the least squares regression has an unique solution). Furthermore,
\begin{align*}
    \mathbb{P}[||\widehat{\theta}^{(M)} - \theta^*||_{\infty} \geq \varepsilon] \leq \delta.
\end{align*}
\label{lem:reg_bounds}
\end{lemma}

We shall now apply the theorem as follows. Denote by $\widehat{\theta}_{i}$ to be the estimate of $\theta^*$ at the beginning of any phase $i$, using all the samples from random explorations in all phases less than or equal to $i-1$.

\begin{remark}
The choice $T_0 := O \left( d^2 \ln^2 \left( \frac{1}{\delta} \right) \right)$ in Equation (\ref{eqn:T_0_defn}) is chosen such that from Lemma \ref{lem:linear_reg_guarantee}, we have that 
\begin{align*}
    \mathbb{P}\left[||\widehat{\theta}^{(\lceil \sqrt{T_0}\rceil)} - \theta^*||_{\infty} \geq \frac{1}{2} \right] \leq \delta
\end{align*}
\label{remark:choice_of_TO}
\end{remark}

\begin{lemma}
Suppose $T_0 = O \left( d^2 \ln^2 \left( \frac{1}{\delta} \right) \right)$ is set according to Equation (\ref{eqn:T_0_defn}). Then, for all phases $i \geq 4$, 
\begin{align}
    \mathbb{P} \left[ || \widehat{\theta}_i - \theta^*||_{\infty} \geq 2^{-i}\right] \leq \frac{\delta}{2^i},
    \label{eqn:cordinates_guarantee}
\end{align}
where $\widehat{\theta}_i$ is the estimate of $\theta^*$ obtained by solving the least squares estimate using all random exploration samples until the beginning of phase $i$.
\label{lem:linear_reg_guarantee}
\end{lemma}
\begin{proof}
The above lemma follows directly from Lemma \ref{lem:reg_bounds}. 
Lemma \ref{lem:reg_bounds} gives that if $\widehat{\theta}_i$ is formed by solving the least squares estimate with at-least $M_i := O \left( \left( 4^i + d\right)\ln \left( \frac{2^i}{\delta} \right) \right)$ samples, then the guarantee in Equation (\ref{eqn:cordinates_guarantee}) holds. 
However, as $T_0 = O \left( (d+1) \ln \left( \frac{2}{\delta} \right) \right)$, we have naturally that $M_i \leq 4^i i \sqrt{T_0}$.
The proof is concluded if we show that at the beginning of phase $i \geq 4$, the total number of random explorations performed by the algorithm exceeds 
$i4^i \lceil \sqrt{T_0} \rceil$.
Notice that at the beginning of any phase $i \geq 4$, the total number of random explorations that have been performed is 
\begin{align*}
    \sum_{j=0}^{i-1}5^i \lceil \sqrt{T_0} \rceil & = \lceil \sqrt{T_0} \rceil \frac{5^i-1}{4},\\
    &\geq i4^i \lceil \sqrt{T_0} \rceil,
\end{align*}
where the last inequality holds for all $i \geq 4$.
\end{proof}

The following corollary follows from a straightforward union bound.
\begin{corollary}
\begin{align*}
    \mathbb{P} \left[ \bigcap_{i \geq 4} || \left\{\widehat{\theta}_i - \theta^*||_{\infty} \leq 2^{-i} \right\} \right] \geq 1 - \delta.
\end{align*}
\label{cor:good_coordinate_behaviour}
\end{corollary}
\begin{proof}
This follows from a simple union bound as follows.
\begin{align*}
        \mathbb{P} \left[ \bigcap_{i \geq 4} \left\{||  \widehat{\theta}_i - \theta^*||_{\infty} \leq 2^{-i} \right\} \right] &= 1 - \mathbb{P} \left[ \bigcup_{i \geq 4}\left\{|| \widehat{\theta}_i - \theta^*||_{\infty} \geq 2^{-i} \right\} \right],\\
        &\geq 1 - \sum_{i \geq 4} \mathbb{P} \left[ ||\widehat{\theta}_i-\theta^*||_{\infty} \geq 2^{-i} \right],\\
        &\geq 1 - \sum_{i\geq 4} \frac{\delta}{2^i},\\
        &\geq 1 - \sum_{i\geq 2}\frac{\delta}{2^i},\\
        &=1-\frac{\delta}{2}.
\end{align*}
\end{proof}

We are now ready to conclude the proof of Theorem \ref{thm:adaptive_dimension}.

\begin{proof}[Proof of Theorem \ref{thm:adaptive_dimension}]

We know from Corollary \ref{cor:good_coordinate_behaviour}, that with probability at-least $1 - \delta$, for all phases $i \geq 4$, we have $||\widehat{\theta}_i-\theta^*||_{\infty} \leq 2^{-i}$.
Call this event $\mathcal{E}$. 
Now, consider the phase $i(\gamma):= \max \left(4,\log_{2} \left( \frac{1}{\gamma} \right)\right)$. 
Now, when event $\mathcal{E}$ holds, then for all phases $i \geq i(\gamma)$, $\mathcal{D}_i$ is the correct set of $d^*$ non-zero coordinates of $\theta^*$. Thus, with probability at-least $1-\delta$, the total regret upto time $T$ can be upper bounded as follows
\begin{align}
    R_T  \leq \sum_{j=0}^{i(\gamma)-1} \left( 25^i T_0 + 5^i \lceil \sqrt{T_0} \rceil \right) & + \sum_{j \geq i(\gamma)}^{\bigg\lceil \log_{25} \left( \frac{T}{T_0} \right) \bigg\rceil} \text{Regret}(\text{OFUL}(1,\delta_i;25^i T_0) \nonumber \\
    & + \sum_{j=i(\gamma)}^{\bigg\lceil \log_{25} \left( \frac{T}{T_0} \right) \bigg\rceil} 5^j \lceil \sqrt{T_0} \rceil.
    \label{eqn:regret_decomposition_dimension}
\end{align}
The term $\text{Regret(OFUL}(L,\delta,T)$ denotes the regret of the OFUL algorithm \cite{oful}, when run with parameters $L \in \mathbb{R}_+$, such that $\|\theta^*\| \leq L$, and $\delta \in (0,1)$ denotes the probability slack and $T$ is the time horizon.
Equation (\ref{eqn:regret_decomposition_dimension}) follows, since the total number of phases is at-most $\bigg\lceil \log_{25} \left( \frac{T}{T_0} \right) \bigg\rceil$.
Standard result from \cite{oful} give us that, with probability at-least $1-\delta$, we have
\begin{align*}
  \text{Regret}(\text{OFUL}(1,\delta; T) \leq 
  4 \sqrt{Td^* \ln \left( 1 + \frac{T}{d^*} \right)}\left(1 + \sigma \sqrt{2 \ln \left( \frac{1}{\delta}\right)+d^*\ln \left( 1 + \frac{T}{d} \right)} \right).
\end{align*}
Thus, we know that with probability at-least $1 - \sum_{i \geq 4}\delta_i \geq 1-\frac{\delta}{2}$, for all phases $i \geq i(\gamma)$, the regret in the exploration phase satisfies
\begin{align}
    \text{Regret}(\text{OFUL}(1,\delta_i;25^i T_0) & \leq  4\sqrt{d^*25^iT_0 \ln \left(1+ \frac{25^iT_0}{d^*}\right)} \nonumber \\
    & \times \left(1 + \sigma\sqrt{2 \ln \left( \frac{2^i}{\delta}\right) + d^* \ln \left( 1+\frac{25^iT_0}{d^*}\right)} \right).
    \label{eqn:intermediate1}
\end{align}

In particular, for all phases $i \in [i(\gamma), \lceil \log_{25}\left( \frac{T}{T_0} \right)]$, with probability  at-least $1-\frac{\delta}{2}$, we have 
\begin{align}
     \text{Regret}(\text{OFUL}(1,\delta_i;25^i T_0) &\leq  4\sqrt{d^*25^iT_0 \ln \left(1+ \frac{25T}{d^*}\right)} \nonumber \\
     & \times \left(1 + \sigma \sqrt{2 \ln \left( \frac{T}{T_0\delta} \right) + d^* \ln \left( 1+\frac{25T}{d^*}\right)} \right), \nonumber\\
     &= \mathcal{C}(T,\delta,d^*) \sqrt{25^iT_0},
     \label{eqn:intermediate_eqn1}
\end{align}
where the constant captures all the terms that only depend on $T$, $\delta$ and $d^*$. We can write that constant as 
\begin{align*}
   \mathcal{C}(T,\delta,d^*) = 4 \sqrt{d^*\ln \left( 1 + \frac{25T}{d^*} \right)} \left(1 + \sigma \sqrt{2 \ln \left( \frac{T}{T_0\delta} \right) + d^* \ln \left( 1+\frac{25T}{d^*}\right)} \right).
\end{align*}

Equation (\ref{eqn:intermediate_eqn1}) follows, by substituting $i \leq \log_{25} \left( \frac{T}{T_0}\right)$ in all terms except the first $25^i$ term in Equation (\ref{eqn:intermediate1}).
As Equations (\ref{eqn:intermediate_eqn1}) and (\ref{eqn:regret_decomposition_dimension}) each hold with probability at-least $1-\frac{\delta}{2}$, we can combine them to get that with probability at-least $1-\delta$,
\begin{align*}
    R_T &\leq 2T_025^{i(\gamma)} + \sum_{j=0}^{\log_{25}\left( \frac{T}{T_0} \right)+1} \mathcal{C}(T,\delta,d^*) \sqrt{25^jT_0} + 25 \lceil\sqrt{T_0} \rceil 5^{\log_{25} \left( \frac{T}{T_0}\right)} ,\\
    &\leq 2T_025^{i(\gamma)} + 25 \sqrt{T} + \mathcal{C}(T,\delta,d^*)\sum_{j=0}^{\log_{25}\left( \frac{T}{T_0} \right)+1}\sqrt{25^j T_0},\\
    &\stackrel{(a)}{\leq} 50 T_0 \frac{2}{{\gamma^{4.65}}} + 25\sqrt{T} + 25\sqrt{T}\mathcal{C}(T,\delta,d^*),\\
    &= O \left( \frac{d^2}{{\gamma^{4.65}}} \ln^2 \left( \frac{1}{\delta} \right) \right) + \widetilde{O} \left( d^*\sqrt{ T \ln\left( \frac{1}{\delta} \right)} \right).
\end{align*}
Step $(a)$ follows from $25 \leq 2^{4.65}$.

\end{proof}

\section{{\ttfamily ALB-Dim} for Stochastic Contextual Bandits with Finite Arms}
\label{appendix-comparision}

\subsection{ALB-Dim Algorithm for the Finite Armed Case}

The algorithm given in Algorithm \ref{algo:main_algo_dimensions_foster}  is identical to the earlier Algorithm \ref{algo:main_algo_dimensions_unknown}, except in Line $8$, this algorithm uses {\ttfamily SupLinRel} \cite{linRel} as opposed to OFUL used in the previous algorithm.
In practice, one could also use {\ttfamily LinUCB} \cite{chu2011contextual} in place of {\ttfamily SupLinRel}. However, we choose to present the theoretical argument using {\ttfamily SupLinRel}, as unlike {\ttfamily LinUCB}, has an explicit closed form regret bound \cite{linRel}.
The pseudocode is provided in Algorithm \ref{algo:main_algo_dimensions_foster}.

In phase $i \in \mathbb{N}$, the {\ttfamily SupLinRel} algorithm is instantiated with input parameter $25^iT_0$ denoting the time horizon, slack parameter $\delta_i \in (0,1)$, dimension $d_{\mathcal{M}_i}$ and feature scaling $b(\delta)$. We explain the role of these input parameters. The dimension ensures that {\ttfamily SupLinRel} plays from the restricted dimension $d_{\mathcal{M}_i}$. The feature scaling implies that when a context $x \in \mathcal{X}$ is presented to the algorithm, the set of $K$ feature vectors, each of which is $d_{\mathcal{M}_i}$ dimensional are $\frac{\phi^{d_{\mathcal{M}_i}}(x,1)}{b(\delta)},\cdots, \frac{\phi^{d_{\mathcal{M}_i}}(x,K)}{b(\delta)}$. The constant $b(\delta) \coloneqq O \left( \tau \sqrt{\log \left( \frac{TK}{\delta} \right)} \right)$ is chosen such that 
\begin{align*}
    \mathbb{P}\left[\sup_{t \in [0,T],a \in \mathcal{A}} \| \phi^M(x_t,a) \|_2 \geq b(\delta) \right] \leq \frac{\delta}{4}.
\end{align*}
Such a constant exists since $(x_t)_{t \in [0,T]}$ are i.i.d. and $\phi^M(x,a)$ is a sub-gaussian random variable with parameter $4\tau^2$, for all $a \in \mathcal{A}$. Similar idea was used in \cite{foster_model_selection}.

\begin{algorithm}[t!]
  \caption{Adaptive Linear Bandit (Dimension) with Finitely Many arms}
  \begin{algorithmic}[1]
 \STATE  \textbf{Input:} Initial Phase length $T_0$ and slack $\delta > 0$.
 \STATE $\widehat{\beta}_0 = \mathbf{1}$, $T_{-1}=0$
 \FOR {Each epoch $i \in \{0,1,2,\cdots\}$}
 \STATE $T_i = 25^{i} T_0$, $\quad$  $\varepsilon_i \gets \frac{1}{2^{i}}$, $\quad$  $\delta_i \gets \frac{\delta}{2^{i}}$
 \STATE $\mathcal{D}_i := \{i : |\widehat{\beta}_i| \geq \frac{\varepsilon_i}{2} \}$
 \STATE $\mathcal{M}_i \coloneqq \inf \{ m : d_m \geq \max \mathcal{D}_i \}$.
 \FOR {Times $t \in \{T_{i-1}+1,\cdots,T_i\}$}
 \STATE Play according to SupLinRel of \cite{linRel} with time horizon of $25^i T_0$ with parameters $\delta_i \in (0,1)$, dimension $d_{\mathcal{M}_i}$ and feature scaling $b(\delta) \coloneqq  O \left( \tau \sqrt{\log \left( \frac{TK}{\delta} \right)} \right)$.
 \ENDFOR
 \FOR {Times $t \in \{T_i+1,\cdots,T_i + 5^i\sqrt{T_0}\}$}
 \STATE Play an arm from the action set ${\mathcal{A}}$ chosen uniformly and independently at random.
 \ENDFOR
 \STATE $\boldsymbol{\alpha}_i \in \real^{S_i \times d}$ with each row being  the arm played during all random explorations in the past.
 \STATE $\boldsymbol{y}_i \in \real^{S_i}$  with $i$-th entry being the observed reward at the $i$-th random exploration in the past
 \STATE $\widehat{\beta}_{i+1} \gets (\boldsymbol{\alpha}_i^T\boldsymbol{\alpha}_i)^{-1}\boldsymbol{\alpha}_i\mathbf{y}_i$, is a $d$ dimensional vector
 \ENDFOR
  \end{algorithmic}
  \label{algo:main_algo_dimensions_foster}
\end{algorithm}

\subsection{Regret Guarantee for Algorithm \ref{algo:main_algo_dimensions_foster}}

In order to specify a regret guarantee, we will need to specify the value of $T_0$. 
We do so as before. 
For any $N$, denote by $\lambda_{max}^{(N)}$ and $\lambda_{min}^{(N)}$ to be the maximum and minimum eigen values of the following matrix: $\boldsymbol{\Sigma}^N := \mathbb{E} \left[ \frac{1}{K} \sum_{j=1}^K \sum_{t=1}^N \phi^M(x_t,j)\phi^M(x_t,j)^T \right]$, where the expectation is with respect to $(x_t)_{t \in [T]}$ which is an i.i.d. sequence with distribution $\mathcal{D}$. First, given the distribution of $x \sim \mathcal{D}$, one can (in principle) compute $\lambda_{max}^{(N)}$ and $\lambda_{min}^{(N)}$ for any $N \geq 1$. Furthermore, from the assumption on $\mathcal{D}$, $\lambda_{min}^{(N)} = \widetilde{O} \left( \frac{1}{\sqrt{d}} \right) > 0$ for all $N \geq 1$.
Choose $T_0 \in \mathbb{N}$ to be the smallest integer such that
\begin{equation}
    \sqrt{T_0} \geq b(\delta)\max \left ( \frac{32\sigma^2}{(\lambda_{min}^{(\lceil \sqrt{T_0} \rceil)})^2}\ln (2d/\delta), \frac{4}{3} \frac{(6\lambda_{max}^{(\lceil \sqrt{T_0}\rceil)}+\lambda_{min}^{(\lceil \sqrt{T_0}\rceil)})(d+\lambda_{max}^{(\lceil \sqrt{T_0}\rceil)})}{(\lambda_{min}^{(\lceil \sqrt{T_0}\rceil)})^2}\ln ( 2d/\delta) \right ).
    \label{eqn:T_0_defn_foster}
\end{equation}
As before, it is easy to see that 
\begin{align*}
    T_0 = O \left( d^2 \ln^2 \left( \frac{1}{\delta}\right) \tau^2 \ln \left( \frac{TK}{\delta}\right) \right).
\end{align*}
Furthermore, following the same reasoning as in Lemmas \ref{lem:linear_reg_guarantee} and \ref{lem:reg_bounds}, one can verify that for all $i \geq 4$, $\mathbb{P} \left[ \| \widehat{\beta}_{i-1} - \beta^* \|_{\infty} \geq 2^{-i} \right] \leq \frac{\delta}{2^i} $.

\begin{theorem}
Suppose Algorithm \ref{algo:main_algo_dimensions_foster} is run with input parameters $\delta \in (0,1)$, and $T_0$ as given in Equation (\ref{eqn:T_0_defn_foster}), then with probability at-least $1-\delta$, the regret after a total of $T$ arm-pulls satisfies
\begin{align*}
R_T \leq 2T_0 \max \left( 25^4, \frac{2}{{\gamma^{4.65}}} \right) + 308 (1 + \ln(2K T \ln T))^{3/2} \sqrt{T d_{m^*}} + 100 \sqrt{T}.
\end{align*}
The parameter $\gamma > 0$ is the minimum  magnitude of the non-zero coordinate of $\beta^*$, i.e., $\gamma = \min \{|\beta^*_i| : \beta^*_i \neq 0 \}$.
\label{thm:adaptive_dimension_foster}
\end{theorem}

In order to parse the above theorem, the following corollary is presented.

\begin{corollary}
Suppose Algorithm \ref{algo:main_algo_dimensions_foster} is run with input parameters $\delta \in (0,1)$, and $T_0 = \widetilde{O} \left(d^2\ln^2 \left( \frac{1}{\delta} \right) \right)$ given in Equation (\ref{eqn:T_0_defn_foster}) , then with probability at-least $1-\delta$, the regret after $T$ times satisfies
\begin{align*}
        R_T &\leq O \left( \frac{d^2}{{\gamma^{4.65}}} \ln^2 ( d/\delta) \tau^2 \ln \left( \frac{TK}{\delta}\right) \right)  + \widetilde{O} (  \sqrt{ T d^*_m}).
\end{align*}
\label{cor:dimension_adaptation_foster}
\end{corollary}

\begin{proof}[Proof of Theorem \ref{thm:adaptive_dimension_foster}]

The proof proceeds identical to that of Theorem \ref{thm:adaptive_dimension}. Observe from Lemmas \ref{lem:reg_bounds} and \ref{lem:linear_reg_guarantee}, that the choice of $T_0$ is such that for all phases $i \geq 1$, the estimate $\mathbb{P}\left[ \| \widehat{\beta}_{i-1} - \beta^* \|_{\infty} \geq 2^{-i} \right] \leq \frac{\delta}{2^i}$.
Thus, from an union bound, we can conclude that 
\begin{align*}
    \mathbb{P} \left[ \cup_{i \geq 4} \| \widehat{\beta}_{i-1} - \beta^* \|_{\infty} \geq 2^{-i}  \right] \leq \frac{\delta}{4}.
\end{align*}
Thus at this stage,  with probability at-least $1-\frac{\delta}{2}$, the following events holds.
\begin{itemize}
    \item $\sup_{t \in [0,T],a \in \mathcal{A}} \| \phi^M(x_t,a) \|_2 \leq b(\delta)$
    \item $ \| \widehat{\beta}_{i-1} - \beta^* \|_{\infty} \leq 2^{-i}$, for all $i \geq 4$.
\end{itemize}
Call these events as $\mathcal{E}$.
As before, let $\gamma > 0$ be the smallest value of the non-zero coordinate of $\beta^*$. Denote by the phase $i(\gamma) \coloneqq \max \left( 4, \log_{2} \left( \frac{2}{\gamma} \right)\right)$. Thus, under the event $\mathcal{E}$, for all phases $i \geq i(\gamma)$, the dimension $d_{\mathcal{M}_i} = d_m^* $, i.e., the SupLinRel is run with the correct set of dimensions.  

It thus remains to bound the error by summing over the phases, which is done identical to that in Theorem \ref{thm:adaptive_dimension}.
With probability, at-least $1-\frac{\delta}{2} - \sum_{i \geq 4}\delta_i \geq 1-\delta $,
\begin{align*}
R_T & \leq \sum_{j=0}^{i(\gamma)-1} \left( 25^jT_0 + 5^j \sqrt{T_0} \right) + \sum_{j=i(\gamma)}^{\bigg\lceil \log_{25} \left( \frac{T}{T_0} \right) \bigg\rceil} \text{Regret(SupLinRel)}(25^iT_0, \delta_i,d_{\mathcal{M}_i,b(\delta)}) \\
& + \sum_{j=i(\gamma)}^{\bigg\lceil \log_{25} \left( \frac{T}{T_0} \right) \bigg\rceil} 5^j\sqrt{T_0},
\end{align*}
where $\text{Regret(SupLinRel)}(25^iT_0, \delta_i,d_{\mathcal{M}_i,b(\delta)}) \leq  44 (1 + \ln(2K 25^iT_0 \ln 25^iT_0))^{3/2}\sqrt{25^iT_0 d_{\mathcal{M}_i}} + 2\sqrt{25^iT_0}$. This expression follows from Theorem $6$ in \cite{linRel}. We now use this to bound each of the three terms in the display above.
Notice from straightforward calculations that the first term is bounded by $2T_025^{i(\gamma)}$ and the last term is bounded above by $25 \lceil \sqrt{T_0} \rceil 5^{\log_{25} \left( \frac{T}{T_0} \right)}$
respectively. We now bound the middle term as
\begin{multline*}
 \sum_{j=i(\gamma)}^{\bigg\lceil \log_{25} \left( \frac{T}{T_0} \right) \bigg\rceil} \text{Reg(SupLinRel)}(25^jT_0, \delta_i,d_{m}^*,b(\delta))  \\ \leq b(\delta) \left( \sum_{j=i(\gamma)}^{\bigg\lceil \log_{25} \left( \frac{T}{T_0} \right) \bigg\rceil} 44 (1 + \ln(2K 25^iT_0 \ln 25^iT_0))^{3/2}\sqrt{25^iT_0 d_{\mathcal{M}_i}} + 2\sqrt{25^iT_0} \right).
 \end{multline*}
 The first summation can be bounded as 
 \begin{align*}
   & \sum_{j=i(\gamma)}^{\bigg\lceil \log_{25} \left( \frac{T}{T_0} \right) \bigg\rceil} 44 (1 + \ln(2K 25^iT_0 \ln 25^iT_0))^{3/2}\sqrt{25^iT_0 d_{\mathcal{M}_i}} \\
   &\leq \sum_{j=i(\gamma)}^{\bigg\lceil \log_{25} \left( \frac{T}{T_0} \right) \bigg\rceil} 44 (1 + \ln(2K T \ln T))^{3/2}\sqrt{25^iT_0 d_{m}^*},\\
   &\leq 44 (1 + \ln(2K T \ln T))^{3/2} 75^{\log_{25} \left( \frac{T}{T_0} \right)} \sqrt{T_0 d_m^*},\\
   &= 308 (1 + \ln(2K T \ln T))^{3/2} \sqrt{T d_m^*},
 \end{align*}
 and the second by 
 \begin{align*}
    \sum_{j=i(\gamma)}^{\bigg\lceil \log_{25} \left( \frac{T}{T_0} \right) \bigg\rceil} 2\sqrt{25^iT_0} \leq 50 \sqrt{T}.
 \end{align*}
 
 Thus, with probability at-least $1-\delta$, the regret of Algorithm \ref{algo:main_algo_dimensions_foster} satisfies
 \begin{align*}
R_T \leq 2T_0 25^{i(\gamma)} + 308 (1 + \ln(2K T \ln T))^{3/2} \sqrt{T d_m^*} + 100 \sqrt{T}, 
 \end{align*}
 where $i(\gamma) \coloneqq \max \left( 4, \log_{2} \left( \frac{2}{\gamma} \right) \right)$. Thus,

     \begin{align*}
R_T \leq 2T_0 \max \left( 25^4, \frac{2}{{\gamma^{4.65}}} \right) + 308 (1 + \ln(2K T \ln T))^{3/2} \sqrt{T d_m^*} + 100 \sqrt{T},
 \end{align*} 
as $25 \leq 2^{4.65}$

\end{proof}

\end{document}